\documentclass[preprint]{colt2020} 
\usepackage{times}

\usepackage[utf8]{inputenc} % allow utf-8 input
\usepackage[T1]{fontenc}    % use 8-bit T1 fonts
\usepackage{hyperref}       % hyperlinks
\usepackage{url}            % simple URL typesetting
\usepackage{booktabs}       % professional-quality tables
\usepackage{amsfonts}       % blackboard math symbols
\usepackage{nicefrac}       % compact symbols for 1/2, etc.
\usepackage{microtype}      % microtypography
\usepackage[english]{babel}
\usepackage{amsmath}
\usepackage{amssymb}
\usepackage{graphicx}
\usepackage{tikz}
\usepackage{fancybox}
\usepackage{soul}
\usepackage{enumitem}
\usepackage{subcaption}
\usepackage{dsfont}
\usepackage{chngpage}
\usepackage{appendix}
\usepackage{capt-of}
\usepackage{newfloat}
\usepackage{mathtools}
\usepackage{relsize}
\usepackage{stackengine}
\usepackage{xfrac}   % \sfrac macro for nicer inline fracs

\coltauthor{%
 \Name{Etienne Boursier} \Email{eboursie@ens-paris-saclay.fr}\\
 \addr Université Paris-Saclay, ENS Paris-Saclay, CNRS, Centre Borelli, Cachan, France.
 \AND
 \Name{Vianney Perchet} \Email{vianney.perchet@normalesup.org}\\
 \addr CREST, ENSAE Paris, Palaiseau, France;\ Criteo AI Lab, Paris, France
}

\newcommand{\ie}{i.e.,\ }
\newcommand{\eg}{e.g.,\ }

\newcommand{\hM}{\widehat{M}}
\newcommand{\hm}{\widehat{m}}

\newcommand{\mP}{\mathbb{P}}
\newcommand{\mE}{\mathbb{E}}
\newcommand{\one}{\mathds{1}}
\newcommand{\hmu}{\widehat{\mu}}

\newcommand{\round}{\mathrm{round}}

\newcommand{\cA}{\mathcal{A}}
\newcommand{\cB}{\mathcal{B}}

\newcommand{\cD}{\mathcal{D}}
\newcommand{\cE}{\mathcal{E}}
\newcommand{\cF}{\mathcal{F}}
\newcommand{\cG}{\mathcal{G}}
\newcommand{\cH}{\mathcal{H}}
\newcommand{\cM}{\mathcal{M}}
\newcommand{\cO}{\mathcal{O}}
\newcommand{\cU}{\mathcal{U}}

\newcommand{\strat}{\mathcal{S}}

\DeclareMathOperator*{\argmax}{arg\,max}
\DeclareMathOperator*{\argmin}{arg\,min}
\DeclareMathOperator*{\card}{\#\!}
\newcommand{\rew}{\mathrm{Rew}}

\newcommand{\smallO}{o}

\newcommand{\algoone}[1][]{\texttt{Selfish-Robust~MMAB}#1\ }
\newcommand{\algotwo}[1][]{\texttt{SIC-GT}#1\ }
\newcommand{\algothree}[1][]{\texttt{RSD-GT}#1\ }
\newcommand{\exploone}[1][]{\texttt{Alternate Exploration}#1\ }

\newcommand{\getrank}[1][]{\texttt{GetRank}#1\ }
\newcommand{\estimateM}[1][]{\texttt{EstimateM}#1\ }
\newcommand{\rsd}[1][]{\texttt{ComputeRSD}#1\ }
\newcommand{\listen}[1][]{\texttt{Listen}#1\ }
\newcommand{\prefsignal}[1][]{\texttt{SignalPreferences}#1\ }
\newcommand{\punishsemi}[1][]{\texttt{PunishSemiHetero}#1\ }
\newcommand{\punishhomo}[1][]{\texttt{PunishHomogeneous}#1\ }
\newcommand{\init}[1][]{\texttt{Initialize}#1\ }
\newcommand{\opt}[1][]{\text{OptArms}#1\ }

\newcommand{\meansignal}[1][]{\texttt{CommPhase}#1\ }
\newcommand{\update}[1][]{\texttt{RobustUpdate}#1\ }
\newcommand{\send}[1][]{\texttt{SendMean}#1\ }
\newcommand{\receive}[1][]{\texttt{ReceiveMean}#1\ }
\newcommand{\signalset}[1][]{\texttt{SignalSet}#1\ }
\newcommand{\sendbit}[1][]{\texttt{SendBit}#1\ }

\makeatletter
\newcounter{protocol}
\newenvironment{protocol}[1][htb]
  {% Update algorithm name
   \let\c@algocf\c@protocol% Update algorithm counter
   \begin{algorithm2e}[#1]%
  }{\end{algorithm2e}}
\makeatother

\theorempostheader{.}
\newtheorem{lemm}{Lemma}
\newtheorem{thm}{Theorem}
\newtheorem{prop}{Proposition}
\newtheorem{defin}{Definition}

\expandafter\let\expandafter\oldproof\csname\string\proof\endcsname
\let\oldendproof\endproof
\renewenvironment{proof}[1][\proofname.]{%
  \par\noindent{\bfseries\upshape #1\ }%
}{\oldendproof}

\title{Selfish Robustness and Equilibria in  Multi-Player Bandits}
 
\begin{document}
\maketitle

\begin{abstract}
Motivated by cognitive radios, stochastic multi-player multi-armed bandits gained a lot of  interest recently. In this class of problems, several players simultaneously pull arms and encounter a collision  -- with 0 reward -- if some of them pull the same arm at the same time.
While the cooperative case where players maximize the collective reward (obediently following some fixed protocol) has been mostly considered, robustness  to malicious players is a crucial  and challenging concern. Existing approaches consider only the case of adversarial \textit{jammers} whose objective is to blindly minimize the collective reward.

We shall consider instead the more natural class of selfish players whose incentives are to maximize their individual rewards, potentially at the expense of the social welfare. We provide the first algorithm robust to selfish players (a.k.a.~Nash equilibrium) with a logarithmic regret, when the arm performance is observed.  When collisions are also observed, \textit{Grim Trigger} type of strategies enable some  implicit communication-based algorithms and we construct robust algorithms in two different settings: the homogeneous (with a regret comparable to the centralized optimal one) and heterogeneous cases (for an adapted and relevant notion of regret). We also provide impossibility results when only the reward is observed or when arm means vary arbitrarily among players.
\end{abstract}

\section{Introduction}

In the classical stochastic Multi Armed Bandit problem (MAB), a player repeatedly chooses among $K$ fixed actions (a.k.a.~arms). After pulling arm $k \in [K] \coloneqq \lbrace 1, \ldots, K \rbrace$, she receives a random reward in $[0,1]$ of mean $\mu_k$. Her goal is  to maximize her cumulative reward up to some horizon $T \in \mathbb{N}$. The performance of a pulling strategy (or algorithm) is assessed by the growth of  \textit{regret},  i.e., the difference between the highest possible expected cumulative reward and the actual cumulative reward.
Since the means $\mu_k$ are unknown beforehand% and only  rewards of  pulled arms are observed
, the player trades off gathering information on under-sampled arms  (exploration) vs.\ using her information (exploitation). Optimal solutions are known in the simplest model \citep{LaiRobbins, Agrawal95, Auer2002}. We refer  to \citep{survey1, survey2, survey3} for an extensive study of MAB.
This simple model captures many sequential decisions problems including clinical trials \citep{thompson33, robbins52} and online recommandation systems \citep{Li2010} and has therefore known a large interest in the past decades.

\medskip

Another classical application of  MAB  is cognitive radios \citep{jouini, anandkumar}. In this context, an arm corresponds to a channel on which a player decides to transmit and the reward is its transmission quality. A key feature of this model, is that it involves several players using channels simultaneously. If several players choose the same arm/channel at the same time, then they \textit{collide} %(their messages interfere with each others)
and receive a null reward. This setting remains somehow simple when a central agent controls simultaneously all players \citep{anantharam, centralized2}, which is far from being realistic.
In reality, the problem is indeed completely decentralized: players are independent, anonymous and cannot communicate to each other. This requires the construction of new algorithms and the development of new techniques dedicated to this \textit{multiplayer bandit} problem. Interestingly, there exist several variants of the base problem, depending on the assumption made on observations/feedback received  \citep{avner14, musicalchair, besson2018, lugosi2018, magesh2019}.
 
More precisely, when  players systematically know whether or not they collide,  this observation actually enables communication between players and a collective regret scaling as in the centralized case is possible, as observed recently \citep{boursier2018, proutiere2019}. Using this idea, it is even possible to asymptotically reach the optimal assignment \citep{GoT2018, tibrewal2019, boursier2019} in the heterogeneous model where the performance of each arm differs among players \citep{kalathil2014, avner15, avner18}. \citet{liu2019} considered the heterogeneous case, when arms also have preferences over players.
%As mentioned in Section~\ref{sec:randomass}, this model provides a good way to deal with selfish players in the heterogeneous case. 
%Je ne comprends pas cette phrase

\medskip

For the aforementioned result to hold,  a crucial (yet sometimes only implicitly stated) assumption is that all players follow cautiously and meticulously some designed protocols and that none of them tries to free-ride the others by acting greedily, selfishly or maliciously. The concern of designing multiplayer bandit algorithms robust to such players has been raised \citep{attar2012}, but only addressed under the quite restrictive assumption of adversarial players called \textit{jammers}. Those try to perturb as much as possible the cooperative players \citep{wang2015, sawant2018, sawant2019}, even if this is extremely costly to them as well.
Because of this  specific objective, they end up using tailored strategies such as only attacking the top channels.

We focus instead on the construction of algorithms with ``good'' regret guarantees even if one (or actually more) selfish player does not follow the common protocol but acts strategically in order to manipulate the other players in the sole purpose of  increasing her own payoff  -- maybe at the cost of other players. This concept appeared quite early in the cognitive radio literature \citep{attar2012}, yet it is still not understood as  robustness to selfish player is intrinsically different (and even non-compatible) with robustness to jammers, as shown in Section~\ref{sec:jamvsgreedy}. In terms of game theory, we aim at constructing ($\varepsilon$-Nash) equilibria in this repeated game with partial observations.

\medskip

The paper is organized as follows.
Section~\ref{sec:model} introduces notions and concepts of selfishness-robust multiplayer bandits and showcases reasons for the design of robust algorithms.
Besides its state of the art  regret guarantees when collisions are not directly observed, \algoone[,]presented in Section~\ref{sec:algo1}, is also robust to selfish players.
In the more complex settings where only the reward is observed or the arm means vary among players, Section~\ref{sec:harder} shows that no algorithm can guarantee both a sublinear regret and selfish-robustness.
The latter case is due to a more general result for random assignments. Instead of comparing the cumulated reward with the best collective assignment in the heterogeneous case, it is then necessary to compare it with a \textit{good} and appropriate suboptimal assignment, leading to the new notion of \textit{RSD-regret}.

When collisions are always observed, Section~\ref{sec:collision} proposes selfish-robust communication protocols. Thanks to this, an adaptation of the work of \citet{boursier2018} is possible to provide a robust algorithm with a collective regret almost scaling as in the centralized case.
In the heterogeneous case, this communication -- along with other new deviation control and punishment protocols --  is also used to provide a robust algorithm with a logarithmic RSD-regret.

\medskip

Our contributions are thus diverse: on top of introducing notions of selfish-robustness, we provide robust algorithms with state of the art regret bounds (w.r.t. non-robust algorithms) in several settings. This is especially surprising when collisions are observed, since it leads to a near centralized regret. Moreover, we show that such algorithms can not be designed in harder settings. This leads to the new, adapted notion of RSD-regret in the heterogeneous case with selfish players and we also provide a \textit{good} algorithm in this case.
These results of robustness are even more intricate knowing they hold against any possible selfish strategy, in contrast to the known results for jammer robust algorithms.

%\paragraph{Section~\ref{sec:model}.} First, notions and concepts of robustness to selfish players are introduced.
%%\vspace{-0.5em}
%\paragraph{Section~\ref{sec:algo1}.} When collisions are not directly observed, \algoone is an algorithm with both satisfying regret guarantees and robustness properties.
%%\vspace{-0.5em}
%\paragraph{Section~\ref{sec:harder}.} We prove that if  players only observe the rewards, designing algorithms with both good regret guarantees and selfishness robustness is actually not possible. 
%
%In the heterogeneous case, this also holds for easier observation settings. This is a consequence of a very well known result of random assignment problem. Instead of comparing the cumulated reward with the best collective assignment, it is then necessary to compare it with a \textit{good} and appropriate suboptimal assignment.
%%\vspace{-0.5em}
%\paragraph{Section~\ref{sec:collision}.} When both collisions and random statistics are always observed, communication protocols between players are possible even in the presence of malicious players. Thanks to this, an adaptation of the work of \citet{boursier2018} is possible to provide a robust algorithm with a collective regret almost scaling as in the centralized case.

%In the heterogeneous case, these communication protocols are also used to provide a robust algorithm having a logarithmic RSD-regret, introduced in Section~\ref{sec:harder}.
\vspace{-0.5em}
\section{Problem statement} \label{sec:model}
In this section, we describe formally the model of multiplayer MAB and introduce concepts and notions of robustness to selfish players (or equilibria concepts).

\vspace{-0.5em}
\subsection{Model}

We denote the transmission qualities of the channels by $(X_k(t))_{1\leq k \leq K} \in [0,1]$,   drawn i.i.d.\ according to  $\nu_k$  of  expectation $\mu_k$. 
In the following,  arm means are assumed to be different and $\mu_{(i)}$ denotes the \mbox{$i$-th} largest mean, \ie $\mu_{(1)} > \mu_{(2)} > \ldots > \mu_{(K)}$.
At each round $t \in [T]$, all $M$ players  simultaneously pull some arms, choice solely based only on their past own observations with $M \leq K$. We denote by  $\pi^j(t)$ the arm played by player $j$, that generates  the reward
\vspace{-0.5em}
\begin{gather*}
r^j(t) \coloneqq X_{\pi^j(t)}(t) \cdot (1-\eta_{\pi^j(t)}(t)), \\\text{where } \eta_k(t) \coloneqq \one\left({\small \card\{j \in [M] \ |\ \pi^j(t) = k\}>1}\right) \text{ is the collision indicator}. 
\end{gather*}
The performance of an algorithm is measured in terms of regret, i.e.,  the difference between the maximal expected  reward and the algorithm cumulative reward after $T$ steps\footnote{As usual, the fact that the horizon $T$ is known is not crucial \citep{degenne2016}.}:
\vspace{-0.5em}
\begin{small}
\begin{equation*}
R_T \coloneqq T  \sum_{k = 1}^{M} \mu_{(k)} - \sum_{t=1}^T\sum_{j =1}^M \mu_{\pi^j(t)}(t) \cdot (1-\eta_{\pi^j(t)}(t)).
\end{equation*}
\end{small}
In multiplayer MAB, three different observation settings are considered.
\vspace{-0.5em}
\begin{description}\itemsep-0.2em
\item[Full sensing:] each player observes both $\eta_{\pi^j(t)}(t)$ and $X_{\pi^j(t)}(t)$ at each round.
\item[Statistic sensing:] each player observes $X_{\pi^j(t)}(t)$ and $r^j(t)$ at each round, \eg the players first sense the quality of a channel before trying to transmit on it.
\item[No sensing:] each player only observes $r^j(t)$ at each round. 
\end{description}

Players are not able to directly communicate to each other, since it involves significant time and energy cost in practice. Some form of communication is still possible between players through observed collisions and has been widely used in recent literature \citep{boursier2018, boursier2019, tibrewal2019, proutiere2019}.

\vspace{-0.5em}
\subsection{Considering selfish players} \label{sec:jamvsgreedy}

As mentioned in the introduction, the literature focused on adversarial malicious players, a.k.a.~\textit{jammers}, while considering selfish players instead of adversarial ones is as (if not more) crucial. 
These two concepts of malicious players are fundamentally different. Jamming-robust algorithms must stop pulling the best arm if it is being jammed. Against this algorithm, a selfish player could therefore pose as a jammer,   always pull the best arm and be left alone on it most of the time.
On the contrary, an algorithm robust to selfish players has to  actually pull this best arm if jammed by some player in order  to ``punish''  her  so that she does not benefit from deviating from the collective strategy.

\medskip

We first introduce some game theoretic concepts before defining notions of robustness. Each player $j$ follows an individual strategy (or algorithm)~$s_j \in \strat$ which determines her action at each round given her past observations.
We denote by $(s_1, \ldots, s_M)=s \in \strat^M$ the strategy profile of all players and by $(s',s_{-j})$ the strategy profile given by $s$ except for the $j$-th player whose strategy is replaced by $s'$. Let $\rew^j_T(s) $ be the cumulative reward of player $j$ when players play the profile $s$.
As usual in game theory, we consider a single selfish player -- even if the algorithms we propose are robust to several selfish players assuming $M$ is known beforehand (its initial estimation  can easily be tricked by several players).
\vspace{-0.25em}
\begin{defin}
A strategy profile $s\in \strat^M$ is  an $\varepsilon$-Nash equilibrium if for any $s' \in \strat$ and $j \in [M]$:
\vspace{-0.5em}
\begin{small}
\begin{equation*}
\mathbb{E}[\rew^j_T(s',s_{-j})] \leq \mathbb{E}[\rew^j_T(s)] + \varepsilon.
\end{equation*}
\end{small}
\end{defin} \vspace{-0.5em}
This simply states that a selfish player wins at most $\varepsilon$ by deviating from $s_j$. We now introduce a  more restrictive property of stability that involves two points: if a selfish player still were to deviate, this would only incur a small loss to other players. Moreover, if the selfish player wants to incur some considerable loss to the collective players (\eg she is adversarial), then she  also has to incur a comparable loss to herself. Obviously,  an $\varepsilon$-Nash equilibrium is $(0, \varepsilon)$-stable.
\vspace{-0.25em}
\begin{defin}
A strategy profile $s \in \strat^M$ is  $(\alpha, \varepsilon)$-stable if for any $s' \in \strat$, $l \in \mathbb{R}_+$ and $i,j \in [M]$:
\vspace{-0.5em}
\begin{small}
\begin{equation*}
\mathbb{E}[\rew^i_T(s',s_{-j})] \leq \mathbb{E}[\rew^i_T(s)] - l \implies \mathbb{E}[\rew^j_T(s',s_{-j})] \leq \mathbb{E}[\rew^j_T(s)] + \varepsilon - \alpha l .\end{equation*}
\end{small}
\end{defin}
\vspace{-1em}
\subsection{Limits of existing algorithms.}

This section explains why existing algorithms  are not robust to selfish players, \ie are not even $\smallO(T)$-Nash equilibria. Besides justifying the design of new appropriate algorithms, this provides some first insights on the way to achieve robustness. 
%Indeed, pointing the disfunctioning methods will suggest good alternatives to them.

\paragraph{Communication between players.}
%\label{sec:introcomm}
Many recent algorithms rely on communication protocols between players to gather their statistics. Facing an  algorithm of this kind, a selfish player would communicate fake statistics to the other players in order to keep the best arm for herself. In case of collision, the colliding player(s) remains unidentified, so a selfish player could modify \textit{incognito} the statistics sent by other players, making them untrustworthy.
A way to make such protocols robust to malicious players is proposed in Section~\ref{sec:collision}. Algorithms relying on communication can then be adapted in the Full Sensing setting.

\paragraph{Necessity of fairness}
%\label{sec:fairness}
An algorithm is  fair if all players asymptotically earn the same reward \textit{a posteriori} and not only in expectation. As already noticed \citep{attar2012}, fairness seems to be a significant criterion in the design of selfish-robust algorithms. Indeed, without fairness, a selfish player tries to always be the one with the largest reward \textit{a posteriori}.

For example, against algorithms attributing an arm among the top-$M$ ones to each player \citep{musicalchair, besson2018, boursier2018}, a selfish player could easily rig the attribution to end with the best arm, largely increasing her individual reward.
Other algorithms work on the basis of first come-first served \citep{boursier2018}. Players first explore and when they detect an arm as both optimal and available, they pull it forever. Such an algorithm is unfair and a selfish player could play more aggressively to end her exploration before the others and to commit on an arm, maybe at the risk of  committing on a suboptimal one (but with high probability on the best arm). The risk taken by the early commit is small compared to the benefit of being the first committing player. As a consequence, these algorithms are not $\smallO(T)$-Nash equilibria.
\vspace{-1em}
\section{Statistic sensing setting}
\label{sec:algo1}

In the statistic sensing setting where $X_k$ and $r_k$ are observed at each round, the \algoone algorithm provides satisfying theoretical guarantees.
\vspace{-0.5em}
\subsection{Description of \algoone}
\vspace{-1em}
\begin{algorithm2e}[h] 
	%\footnotesize
    \DontPrintSemicolon
    \KwIn{$T, \gamma_1\coloneqq \frac{13}{14}, \gamma_2\coloneqq \frac{16}{15}$}
    $\beta \gets 39$;\quad $\hM, t_m \gets \estimateM(\beta, T)$ \;
 	%$\gamma_1 \gets $ and $\gamma_2 \gets $\;
    Pull $k \sim \mathcal{U}(K)$ until round $\frac{\gamma_2}{\gamma_1} t_m$ \tcp*{first waiting room}
	$j \gets \getrank(\hM, t_m, \beta, T)$ and
	pull $j$ until round $\left(\frac{\gamma_2}{\gamma_1^2 \beta^2 K^2} + \frac{\gamma_2^2}{\gamma_1^2}\right)\ t_m$ \; %\tcp*{waiting room}
	Run $\exploone(\hM, j)$ until $T$\;
    \caption{\label{algo:algo1}\algoone}
\end{algorithm2e}
\vspace{-0.5em}
A global description of \algoone is given by Algorithm~\ref{algo:algo1}. The pseudocodes of \estimateM[,]\getrank and \exploone are respectively given by Protocols~\ref{proto:estimM}, \ref{proto:getrank} and Algorithm~\ref{algo:explo1} in Appendix~\ref{app:algo1} due to space constraints.

\estimateM and \getrank respectively estimate the number of players $M$ and attribute ranks in $[M]$ among the players. They form the initialization phase, while \exploone optimally balances between exploration and exploitation.
\vspace{-0.5em}
\subsubsection{Initialization phase}
\label{sec:estimm}
Let us first introduce the following quantities:\vspace{-0.5em}
\begin{itemize}\itemsep0em
\item $N_k^j(t) =\lbrace t' \leq t \ | \ \pi^j(t')=k \text{ and } X_k(t') >0 \rbrace$ are  rounds when player $j$ observed $\eta_k$.
\item $C_k^j(t) = \lbrace t' \in N_k^j(t) \ | \ \eta_k(t') = 1 \rbrace$ are  rounds when player $j$ observed a collision.
\item $\hat{p}_k^j(t) = \sfrac{\card C_k^j(t) }{\card N_k^j(t)}$ is the empirical probability to collide on the arm $k$ for player~$j$.
\end{itemize}

% ref here:  Lemma~\ref{lemma:estim1}.

During the initialization, the players estimate $M$ with large probability as given by Lemma~\ref{lemma:estim1} in Appendix~\ref{app:algo1descr}. Players first pull uniformly at random in $[K]$. As soon as $\card N_k^j \geq n$ for any $k \in [K]$ and some fixed $n$, player $j$ ends the \estimateM protocol and estimates $\hM$ as the closest integer to $1 + \log(1-\sfrac{\sum_k \hat{p}_k^j(t_M)}{K})/\log( 1- \frac{1}{K})$.
This estimation procedure is the same as the one of \citet{musicalchair}, except for the following features:

 i) Collisions indicators are not always observed, as we consider statistic sensing here. For this reason, the number of observations of $\eta_k$ is random. The stopping criterion $\min_k \card N_k^j(t) \geq n$ ensures that players don't need to know $\mu_{(K)}$ beforehand, but they also do not end \estimateM simultaneously. This is why a \textit{waiting room} is needed, during which a player continues to pull uniformly at random to ensure that all players are still pulling uniformly at random if some player is still estimating $M$.
 
ii)  The collision probability is not averaged over all arms, but estimated for each arm individually, then averaged. This is necessary for robustness as explained in Appendix~\ref{app:algo1}, despite making the estimation longer.

%ref here: Lemma~\ref{lemma:getrank}.

\paragraph{Attribute ranks.} After this first procedure, players then proceed to a \textit{Musical Chairs} \citep{musicalchair} phase to attribute ranks among them as given by Lemma~\ref{lemma:getrank} in Appendix~\ref{app:algo1descr}. Players sample uniformly at random in $[M]$ and stop on an arm $j$ as soon as they observe a positive reward. The player's rank is then $j$ and only attributed to her. Here again, a \textit{waiting room} is required to ensure that all players are either pulling uniformly at random or only pulling a specific arm (corresponding to their rank) during this procedure. During this second waiting room, a player thus pulls the arm corresponding to her rank.

\vspace{-0.5em}
\subsubsection{Exploration/exploitation}

% ref here: Lemma~\ref{lemma:ucb1}.

After the initialization,  players  know $M$ and have different ranks. They enter the second phase, where they follow \exploone[,]inspired by \citet{proutiere2019}. Player $j$ sequentially pulls arms in $\cM^j(t)$, which is the ordered list of her $M$ best empirical arms, unless she has to pull her $M$-th best empirical arm. In that case, she instead chooses at random between actually pulling it or pulling an arm to explore (any arm not in $\cM^j(t)$ with an upper confidence bound larger than the $M$-th best empirical mean, if there is any). %Any arm is thus explored as much as required.

Since  players proceed in a shifted fashion, they never collide when $\cM^j(t)$ are the same for all~$j$. Having different $\cM^j(t)$ happens in expectation a constant (in $T$) amount of times, so that the contribution of collisions  to the  regret is  negligible.

\vspace{-0.5em}
\subsection{Theoretical results}

This section provides theoretical guarantees of \algoone[.]Theorem~\ref{thm:algoone} first presents guarantees in terms of regret. Its proof is given in Appendix~\ref{app:regretalgo1}.
\begin{thm}
\label{thm:algoone}
The collective regret of \algoone is bounded as
\begin{small}
\begin{equation*}
\mE[R_T] \leq M \sum_{k>M} \frac{\mu_{(M)} - \mu_{(k)}}{\mathrm{kl}(\mu_{(k)}, \mu_{(M)})}\log(T) + \cO \left( \frac{MK^3}{\mu_{(K)}} \log(T)\right).
\end{equation*}
\end{small}
\end{thm} 
It can also be noted from Lemma~\ref{lemma:ucb1} in Appendix~\ref{app:regretalgo1} that the regret due to \exploone is $M \sum_{k>M} \frac{\mu_{(M)} - \mu_{(k)}}{\mathrm{kl}(\mu_{(k)}, \mu_{(M)})}\log(T) + \smallO(\log(T))$, which is known to be optimal for algorithms using no collision information \citep{besson2019}. \exploone thus gives an optimal algorithm under this constraint, if $M$ is already known and ranks already attributed (as the $\cO(\cdot)$ term in the regret is the consequence of their estimation).

On top of  good regret guarantees, \algoone is robust to selfish behaviors as highlighted by Theorem \ref{thm:robust1} (whose proof is deterred to  Appendix~\ref{app:statisticrobust1}).

%\vspace{1em}

\begin{thm}\label{thm:robust1} Playing \algoone is an $\varepsilon$-Nash equilibrium and is ${(\alpha, \varepsilon)\text{-stable}}$ 
%\vspace{-0.5em}
\begin{equation*}
\text{with} \quad {\textstyle \varepsilon = \sum_{k>M} \frac{\mu_{(M)} - \mu_{(k)}}{\mathrm{kl}(\mu_{(k)}, \mu_{(M)})}\log(T) + \cO\left(\frac{\mu_{(1)}}{\mu_{(K)}} K^3 \log(T)\right)} \quad \text{and} \quad {\textstyle
\alpha = \frac{\mu_{(M)}}{\mu_{(1)}}}.
\end{equation*}
\end{thm}

These points are proved for an \textit{omniscient} selfish player (knowing all the parameters beforehand). This is a very strong assumption and a real player would not be able to win as much by deviating from the collective strategy. 
Intuitively, a selfish player would need to explore sub-optimal arms as given by the known individual lower bounds. However, a selfish player can actually decide to not explore but deduce the exploration of other players from collisions.

\vspace{-1em}
\section{On harder problems}
\label{sec:harder}
\vspace{-0.5em}
Following the positive results of the previous section (existence of robust algorithms) in the homogeneous case with statistical sensing, we now provide in this section  impossibility results for both no sensing and heterogeneous cases.
By showing its limitations, it also suggests a proper way to consider the heterogeneous problem in the presence of selfish players.
\vspace{-0.5em}
\subsection{Hardness of no sensing setting}
%
%The no sensing setting corresponds to the case the players only observe $r^j(t)$ at each time step, meaning they do not know whether they receive a $0$ because of collision or because of a bad arm ($X_k=0$).

\begin{thm}
\label{thm:nosensing1}
In the no sensing setting, there is no profile of strategy $s$ such that,  for all problem parameters $(M, \pmb{\mu})$, $ \mE[R_T]= \smallO(T)$ and  $s$ is an $\varepsilon(T)$-Nash equilibrium with $ \varepsilon(T) =\smallO(T)$.
\end{thm}
\begin{proof}
Consider a strategy $s$ verifying the first property and a problem instance $(M, \pmb{\mu})$ where the selfish player only pulls the best arm. Let $\pmb{\mu'}$ be the mean vector $\pmb{\mu}$ where $\mu_{(1)}$ is replaced by $0$. Then, because of the considered observation model, the cooperative players can not distinguish the two worlds $(M, \pmb{\mu})$ and $(M-1, \pmb{\mu'})$. Having a sublinear regret in the second world implies $\smallO(T)$ pulls on the arm $1$ for the cooperative players. So in the first world, the selfish player will have a reward in $\mu_{(1)} T - \smallO(T)$, which is thus a linear improvement in comparison with following $s$ if $\mu_{(1)} > \mu_{(2)}$.
\end{proof}

Theorem~\ref{thm:nosensing1} is proved for a selfish players who knows the means $\pmb{\mu}$ beforehand, as the notion of Nash equilibrium prevents against any possible strategy, which includes committing to an arm for the whole game. The knowledge of $\pmb{\mu}$ is actually not needed, as a similar result holds for a selfish player committing to an arm chosen at random when the best arm is $K$ times better than the second one.
The question of existence of robust algorithms remains yet open if we restrict selfish strategies to more \textit{reasonable} algorithms.

%\vspace{-1.5em}
\subsection{Heterogeneous model}

We consider the full sensing heterogeneous model, where player $j$ receives the reward~$r^j(t) \coloneqq X_{\pi^j(t)}^j(t) (1-\eta_{\pi^j(t)})$ at round~$t$, with $X_k^j \overset{\text{\tiny i.i.d.}}{\sim} \nu_k^j$ of mean $\mu_k^j$. The arm means here vary among the players. This models that transmission quality depends on individual factors such as the localization.
% parler en intro de ce probleme et de la litterature s'y intéressant

\subsubsection{A first impossibility result}
%In order to analyze this model, we can start by wondering how to proceed if all players know their own reward distributions.
\begin{thm}
\label{thm:heterimposs}
If the regret is compared with the optimal assignment, there is no strategy $s$  such that, for all problem parameters $\pmb{\mu}$, $ \mE[R_T] =\smallO(T)$ and $s$ is an $\varepsilon(T)$-Nash equilibrium with $\varepsilon(T) = \smallO(T)$.
\end{thm}

\begin{proof}
Assume $s$ satisfies these properties and consider a problem instance $\pmb{\mu}$ such that the selfish player unique best arm $j_1$ has mean $\mu^j_{(1)} = 1/2$ and the difference between the optimal assignment utility and the utility of the best one assigning arm $j_1$ to $j$ is $1/3$.

Such an instance is of course possible. Consider a selfish player $j$ playing exactly the strategy $s_j$ but as if her reward vector $\pmb{\mu^j}$ was actually $\pmb{\mu'^j}$ where $\mu_{(1)}^j$ is replaced by $1$ and all other $\mu_k^j$ by $0$, \ie she fakes a second world $\pmb{\mu'}$ in which the optimal assignment gives her the arm $j_1$. In this case, the sublinear regret assumption of $s$ implies that player $j$ pulls $j_1$ a time $T-\smallO(T)$, while in the true world, she would have pulled it $\smallO(T)$ times.
She thus earns an improvement at least $(\mu^j_{(1)} - \mu^j_{(2)}) T - \smallO(T)$ w.r.t. playing $s_j$, contradicting the Nash equilibrium assumption.
\end{proof}
\vspace{-0.5em}

\subsubsection{Random assignments}
\label{sec:randomass}

We now take a step back and describe ``relevant'' allocation procedures for the heterogeneous case, when the vector of  means $\pmb{\mu^j} $ is already known by player $j$. 

An assignment is  \textit{symmetric} if, when  $\pmb{\mu^j} = \pmb{\mu^i}$,  players $i$ and $j$ get the same \textbf{expected} utility, \ie no player is \textit{a priori} favored\footnote{The concept of fairness introduced above is stronger, as no player should be \textit{a posteriori} favored.}.  It is \textit{strategyproof} if being truthful is a dominant strategy for any player and \textit{Pareto optimal} if the social welfare (sum of utilities) can not be improved without hurting any player.
Theorem~\ref{thm:heterimposs} is a consequence of Theorem~\ref{thm:zhou} below. 
\begin{thm}[\citealt{zhou1990}]\label{thm:zhou}
For $M\geq 3$, there is no symmetric, Pareto optimal and strategyproof random assignment algorithm.
\end{thm}
\citet{liu2019} circumvent this assignment problem with  player-preferences for arms. Instead of assigning a player to a contested arm, the latter decides who  gets to pull it, following its preferences. 

In the case of random assignment, \citet{abdulkadiroglu1998} proposed the Random Serial Dictatorship (RSD) algorithm, which is symmetric and strategyproof. The algorithm is rather simple: pick uniformly at random an ordering of the $M$ players. Following this order, the first player picks her preferred arm, the second one her preferred remaining arm and so on. \citet{svensson1999} justified the choice of RSD for symmetric strategyproof assignment algorithms. \citet{adamczyk2014} recently studied  efficiency ratios of such assignments:  if $U_{\max}$ denotes the expected social welfare  of the optimal assignment, the expected social welfare of RSD is greater than $U^2_{\max}/eM$ while no strategyproof algorithm can guarantee more than $U^2_{\max}/M$. As a consequence, RSD is optimal up to a (multiplicative) constant and will serve as a benchmark in the remaining.

Indeed, instead of defining the regret in comparison with the optimal assignment as done in the classical heterogeneous multiplayer bandits, we are going to define it with respect to RSD to incorporate  strategy-proofness constraints. Formally,  the RSD-regret is defined as:
\begin{small}
\begin{equation*}
R^{\text{RSD}}_T \coloneqq T{\mathlarger \mE_{\sigma \sim \cU\left(\mathfrak{S}_M\right)} } \bigg[ \sum_{k = 1}^{M} \mu_{\pi_\sigma(k)}^{\sigma(k)}\bigg] -  \sum_{t=1}^T \sum_{j =1}^M \mu^j_{\pi^j(t)}(t) \cdot (1-\eta_{\pi^j(t)}(t)) ,
\end{equation*}
\end{small}with $\mathfrak{S}_M$ the set of permutations over $[M]$ and $\pi_\sigma(k)$ the arm attributed by RSD to player $\sigma(k)$ when the order of dictators is $(\sigma(1), \ldots, \sigma(M))$. Mathematically, $\pi_\sigma$ is defined by:
\begin{small}
\begin{equation*}
\pi_\sigma(1) = \argmax_{l \in [M]} \mu_l^{\sigma(1)} \quad \text{and} \quad \pi_\sigma(k+1) = \argmax\limits_{\substack{l \in [M] \\ l \not\in \lbrace  \pi_\sigma(l') \ | \ l' \leq k\rbrace}} \mu_l^{\sigma(k+1)}.
\end{equation*}
\end{small}
\vspace{-1cm}
%\algothree in Section~\ref{sec:algo3} incurs only a logarithmic RSD-regret.
\section{Full sensing setting} 
\label{sec:collision}
This section focuses on the full sensing setting, where both $\eta_k(t)$ and $X_k(t)$ are always observed as we proved impossibility results for more complex settings. As mentioned before, recent algorithms leverage the observation of collisions to enable some communication between players by forcing them. Some of these communication protocols can be modified to allow robust communication.
This section is structured as follows. First, insights on two new protocols are given for  robust communications. Second, a robust adaptation of SIC-MMAB is given, based on these two protocols. Third, they can also be used to reach a logarithmic RSD-regret in the heterogeneous case.

\subsection{Making communication robust}

To have robust communication, two new complementary protocols are needed. The first one allows to send messages between players and to detect when they have been corrupted by a malicious player. 
If this has been the case, the players then use the second protocol to proceed to a collective punishment, which forces every player to suffer a considerable loss for the remaining of the game. 
Such punitive strategies are called ``Grim Trigger'' in game theory and are used to deter  defection in repeated games \citep{friedman1971, axelrod1981, fudenberg2009}.

\subsubsection{Back and forth messaging}
\label{sec:backforth}

% ici c'est une "preuve" de la robustness des communications (ie, on ne prouve rien de plus dessus en annexe)

Communication protocols in the collision sensing setting usually rely on the fact that collision indicators can be seen as bits sent from a player to another one as follows. If player $i$ sends a binary message $m_{i \to j}=(1, 0, \ldots, 0, 1)$ to player $j$ during a predefined time window, she proceeds to the sequence of pulls $(j, i, \ldots, i, j)$, meaning she purposely collides with $j$ to send a $1$ bit (reciprocally, not colliding corresponds to a $0$ bit).
A malicious player trying to corrupt a message can only create new collisions, \ie replace zeros by ones. The key point is that the inverse operation is not possible. 

If player $j$ receives the (potentially corrupted) message $\hm_{i \to j}$, she repeats it to player~$i$. This second message can also be corrupted by the malicious player and player~$i$ receives~$\widetilde{m}_{i \to j}$. However, since the only possible operation is to replace zeros by ones, there is no way to transform back $\hm_{i \to j}$ to $m_{i \to j}$ if the first message had been corrupted.
The player~$i$ then just has to compare $\widetilde{m}_{i \to j}$ with $m_{i \to j}$ to know whether or not at least one of the two messages has been corrupted. We call this protocol \textit{back and forth} communication.

\medskip

In the following, other malicious communications are possible. Besides sending false information (which is managed differently), a malicious player can send different  statistics to the others, while they need to have  the exact same statistics. To overcome this issue, players will send to each other statistics  sent to them by any player. If two players have received different statistics by the same player, at least one of them automatically realizes it. 

\subsubsection{Collective punishment}
\label{sec:punish}

% ici on a des résultats théoriques données par Lemmas~\ref{lemma:punishment}, \ref{lemma:punishhomo} and \ref{lemma:punishhetero}. (le dernier à ref dans la partie semi hetero plutot)

The back and forth protocol detects if a malicious player interfered in a communication  and, in that case, a collective punishment is triggered (to deter defection). The malicious player is yet unidentified and can not be specifically targeted. The  punishment thus guarantees that the average reward earned by any player is smaller than the average reward of the algorithm, $\bar{\mu}_M\coloneqq \frac{1}{M} \sum_{k=1}^M \mu_{(k)}$.

A naive way to \textit{punish} is to pull all arms uniformly at random. The selfish player then gets the reward $(1-1/K)^{M-1} \mu_{(1)}$ by pulling the best arm, which can be larger than $\bar{\mu}_M$. A good punishment should therefore pull arms more often the better they are.

\medskip

During the punishment, players pull each arm $k$ with probability $1 - \big(\gamma\frac{\sum_{l=1}^M \hmu^j_{(l)}(t)}{M \hmu^j_k(t)}\big)^{\frac{1}{M-1}}$ at least, where $\gamma=\left(1 - 1/K\right)^{M-1}$. Such a strategy is possible as shown by Lemma~\ref{lemma:punishment} in Appendix~\ref{app:punish}. Assuming the arms are correctly estimated, i.e.,  the expected reward a selfish player gets by pulling $k$ is approximately
$\mu_k (1-p_k)^{M-1}$, with ${p_k = \max \Big( 1 - \big(\gamma\frac{\bar{\mu}_M}{\mu_k}\big)^{\frac{1}{M-1}}, 0\Big)}$.

If $p_k=0$,  then $\mu_k$ is smaller than $\gamma \bar{\mu}_M$ by definition; otherwise, it necessarily holds that~$\mu_k (1-p_k)^{M-1} = \gamma \bar{\mu}_M$. As a consequence, in both cases, the selfish player earns at most $\gamma \bar{\mu}_M$, which involves a relative positive decrease of $1-\gamma$ in reward w.r.t. following the cooperative strategy. More details on this protocol are given by Lemma~\ref{lemma:punishhomo} in Appendix~\ref{app:greedyproofsicmmab}.

\subsection{Homogeneous case: \algotwo}
\label{sec:algo2}

In the homogeneous case, these two protocols can be incorporated in the SIC-MMAB algorithm of \citet{boursier2018} to provide \algotwo[,]which is robust to selfish behaviors and still ensures a regret comparable to the centralized lower bound.

\citet{boursier2019} recently improved the communication protocol by choosing a leader and communicating all the information only to this leader. A malicious player would do anything to be the leader. \algotwo avoids such a behavior by choosing two leaders who either agree or trigger the punishment. More generally with $n+1$ leaders, this protocol is robust to $n$ selfish players.
The detailed algorithm is given by Algorithm~\ref{alg:algo2} in Appendix~\ref{app:sicmmab_descript}.

\paragraph{Initialization.} %First an initialization phase is required. 
The original initialization phase of SIC-MMAB has a small regret term, but it is not robust. During the initialization, the players here pull uniformly at random to estimate $M$ as in \algoone and then attribute ranks the same way. The players with ranks $1$ and $2$ are then leaders. Since the collision indicator is always observed here, this estimation can be done in an easier and better way. 
The observation of $\eta_k$ also enables players to remain synchronized after this phase as its length does not depend on unknown parameters.

\paragraph{Exploration and Communication.} Players  alternate between exploration and communication once the initialization is over. During the $p$-th exploration phase, each arm still requiring exploration is pulled $2^p$ times by every player in a collisionless fashion. Players~then communicate to each leader their empirical means in binary after every exploration phase, using the back and forth trick explained in Section~\ref{sec:backforth}. Leaders then check that~their information match. If some undesired behavior is detected, a collective punishment is~triggered. 

Otherwise, the leaders determine the sets of optimal/suboptimal arms and send them to everyone.
To prevent the selfish player from sending fake statistics, the leaders gather the empirical means of all players, except the extreme ones (largest and smallest) for every arm. If the selfish player sent outliers, they are thus cut out from the collective estimator, which is thus the average of $M-2$ individual estimates. This estimator can be biased by the selfish player, but a concentration bound given by Lemma~\ref{lemma:concentration2} in Appendix~\ref{app:proofsicmmabexploregret} still holds.

\paragraph{Exploitation.} As soon as an arm is detected as optimal, it is pulled until the end. To ensure fairness of \algotwo[,]players will actually rotate over all the optimal arms so that none of them is favored. This point is thoroughly described in Appendix~\ref{app:sicmmab_descript}.
Theorem~\ref{thm:sicmmab1}, proved in Appendix~\ref{app:sicmmab}, gives theoretical results for \algotwo[.]
\begin{thm}
\label{thm:sicmmab1}
Define $\alpha = \frac{1 - (1-1/K)^{M-1}}{2} $ and assume $M\geq 3$.
\begin{enumerate}\itemsep0em
\item The collective regret of \algotwo is bounded as
\begin{small}
\begin{equation*}
\mE[R_T] \leq \cO\bigg( \sum_{k>M} \frac{\log(T)}{\mu_{(M)} - \mu_{(k)}} + MK^2 \log(T) + M^2K \log^2\Big( \frac{\log(T)}{(\mu_{(M)} - \mu_{(M+1)})^2} \Big) \bigg).
\end{equation*}
\end{small}
\item Playing \algotwo is an $\varepsilon$-Nash equilibrium and is $(\alpha, \varepsilon)$-stable with
\begin{small}
\begin{equation*}
\varepsilon =  \cO\bigg(\sum_{k>M}\frac{\log(T)}{\mu_{(M)} - \mu_k} + K^2 \log(T) + MK \log^2\Big(\frac{\log(T)}{(\mu_{(M)}-\mu_{(M+1)})^2}\Big) +  \frac{K \log(T)}{\alpha^2 \mu_{(K)}}  \bigg).
\end{equation*}
\end{small}
\end{enumerate}
\end{thm}
\subsection{Semi-heterogeneous case: \algothree}
\label{sec:algo3}

% ref Lemma~\ref{lemma:punishhetero}

The punishment strategies described above can not be extended to the heterogeneous case, as the relevant probability of choosing each arm would depend on the  preferences of the malicious player which are unknown (even  her identity might not be discovered). Moreover, as already explained in the homogeneous case, pulling each arm uniformly at random is not an appropriate punishment strategy\footnote{Unless in the specific case where $\mu^j_{(1)}(1-1/K)^{M-1} < \frac{1}{M}\sum_{k=1}^M \mu^j_{(k)}$.}.
We therefore consider  the $\delta$-heterogeneous setting, which allows punishments for small values of $\delta$ as given by Lemma~\ref{lemma:punishhetero} in Appendix~\ref{app:rsdgreedyproof}.
The heterogeneous model was justified by the fact that transmission quality depends on individual factors such as localization.
The $\delta$-heterogeneous assumption relies on the idea that such individual factors are of a different order of magnitude than global factors (as the availability of a channel). As a consequence, even if arm means differ from player to  player, these variations remain relatively small.
\begin{defin}
The setting is  $\delta$-heterogeneous if there exists $\{\mu_k ; k \in [K]\}$ such that for all $j$ and $k$, $\mu_k^j \in [(1-\delta)\mu_k, (1+\delta)\mu_k]$.
\end{defin}
In the semi-heterogeneous full sensing setting, \algothree provides a robust, logarithmic RSD-regret algorithm. Its complete description is given by Algorithm~\ref{alg:algo3} in Appendix~\ref{app:rsd_descript}. 
\vspace{-0.5em}
\subsubsection{Algorithm description}

\algothree starts with the exact same initialization as \algotwo to estimate $M$ and attribute ranks among the players.
The time is then divided into superblocks which are divided into $M$ blocks. During the $j$-th block of a superblock, the dictators ordering\footnote{The ordering is actually $(\sigma(j), \ldots, \sigma(j-1))$ where $\sigma(j)$ is the player with rank $j$ after the initialization. For sake of clarity, this consideration is omitted here.} is $(j, \ldots, M, 1, \ldots, j-1)$. Moreover, only the $j$-th player can send messages during this block.

\paragraph{Exploration.} The exploring players pull sequentially all the arms. Once player $j$ knows her $M$ best arms and their ordering, she waits for a block $j$ to initiate communication.

\paragraph{Communication.} Once a player starts a communication block, she proceeds in three successive steps as follows:
\vspace{-0.25em}
\begin{enumerate}\itemsep0em
\item she first collides with all players to signal the beginning of a communication block. The other players then enter a listening state, ready to receive messages.
\item She then sends to every player her ordered list of $M$ best arms. Each player then repeats this list to detect the potential intervention of a malicious player.
\item Finally, any player who detected the intervention of a malicious player signals to everyone the beginning of a collective punishment.
\end{enumerate}\vspace{-0.25em}
After a communication block $j$, every one knows the preferences order of player $j$, who is now in her exploitation phase, unless a punishment protocol has been started.

\paragraph{Exploitation.} While exploiting, player $j$ knows the preferences of all other exploiting players. Thanks to this, she can easily compute the arms attributed by the RSD algorithm between the exploiting players, given the dictators ordering of the block. %Collisions with exploring players still occur. Another way would be to give the priority to the exploring players, but in some cases then, it could be interesting for a malicious player to not leave the exploration phase.

Moreover, as soon as she collides in the beginning of a block while not intended (by her), this means an exploring player is starting a communication block. The exploiting player then starts listening to the arm preferences of the communicating player.

\subsubsection{Theoretical guarantees}

Here are some insights to understand how \algothree reaches the utility of the RSD algorithm, which are rigorously detailed by Lemma~\ref{lemma:rsdmatching} in Appendix~\ref{app:rsdgreedyproof}.
% Assume all players are exploiting. 
With no malicious player, the players ranks given by the initialization provide a random permutation $\sigma \in \mathfrak{S}_M$ of the players and always considering the dictators ordering $(1, \ldots, M)$ would lead to the expected reward of the RSD algorithm. However, a malicious player can easily rig the initialization to end with rank $1$. In that case, she largely improves her individual reward w.r.t. following the cooperative strategy.

To avoid such a behavior, the dictators ordering should rotate over all permutations of $\mathfrak{S}_M$, so that the rank of the player has no influence. However, this leads to an undesirable combinatorial $M!$ dependency of the regret. 
\algothree instead rotates over the dictators ordering $(j, \ldots,M,1,\ldots, j-1)$ for all $j \in [M]$. If we note $\sigma_0$ the $M$-cycle $(1 \ldots M)$, the considered permutations during a superblock are of the form $\sigma \circ \sigma_0^{-m}$ for $m \in [M]$.
The malicious player $j$ can only influence the distribution of $\sigma^{-1}(j)$: assume w.l.o.g.\ that $\sigma(1)=j$. The permutation $\sigma$ given by the initialization then follows the uniform distribution over $\mathfrak{S}_{M}^{j \to 1} = \lbrace \sigma \in \mathfrak{S}_M \ | \ \sigma(1)=j \rbrace$. But then, for any $m \in [M]$, $\sigma \circ \sigma_0^{-m}$ has a uniform distribution over $\mathfrak{S}_{M}^{j \to 1+m}$. In average over a superblock, the induced permutation still has a uniform distribution over $\mathfrak{S}_{M}$. So the malicious player  has no interest in choosing a particular rank during the initialization, making the algorithm robust.

\medskip

Thanks to this remark and robust communication protocols, \algothree possesses theoretical guarantees given by Theorem~\ref{thm:rsd} (whose proof is deterred to Appendix~\ref{app:rsd}).
\vspace{-0.5em}
\begin{thm}
\label{thm:rsd}
Consider the $\delta$-heterogeneous setting and define $r = \frac{1-\left( \frac{1+\delta}{1-\delta}\right)^2 (1-1/K)^{M-1}}{2}$ and ${\Delta = \min\limits_{(j,k) \in [M]^2} \mu_{(k)}^j - \mu_{(k+1)}^j}$.
\begin{enumerate}
\item The RSD-regret of \algothree is bounded as:
$
\mE[R_T^{\mathrm{RSD}}] \leq \cO\big(MK \Delta^{-2} \log(T) + MK^2 \log(T)\big).
$
\item If $r>0$, playing \algothree is an $\varepsilon$-Nash equilibrium and is $(\alpha, \varepsilon)$-stable with
\begin{itemize}\itemsep0em
\item $\varepsilon = \cO\Big(\frac{K \log(T)}{\Delta^2}  + K^2 \log(T) + \frac{K\log(T)}{(1-\delta)r^2 \mu_{(K)}}\Big),$
\item $\alpha = \min\Big(r\left( \frac{1+\delta}{1-\delta}\right)^3 \frac{\sqrt{\log(T)}-4M}{\sqrt{\log(T)}+4M},\quad \frac{\Delta}{(1+\delta)\mu_{(1)}},\quad  \frac{(1-\delta)\mu_{(M)}}{(1+\delta)\mu_{(1)}}\Big).$
%\item $\alpha = \min\left(\frac{\Delta}{M\mu_{(1)} (1+\delta)}, \left(\frac{1-\delta}{1+\delta}\right)^3 r \right)(1-1/\sqrt{\log(T)}).$
\end{itemize}
\end{enumerate}
\end{thm}
\vspace{-1em}
\section{Conclusion}

We introduced notions of robustness to selfish players and provided impossibility results in hard settings.
With statistic sensing, \algoone gives a rather simple robust and efficient algorithm, besides being optimal among the class of algorithms using no collision information.
On the other hand when collisions are observed, robust algorithms relying on communication through collisions are possible. 
Thanks to this, even selfish-robust algorithms can achieve near centralized regret in the homogeneous case, which is not intuitive at first sight.
In the heterogeneous case, a new adapted notion of regret is introduced and \algothree achieves a good performance with respect to it.

\medskip

\algothree heavily relies on collision observations and future work should focus on designing a comparable algorithm in both performance and robustness without this feature. 
The topic of robustness to selfish players in multiplayer bandits still remains largely unexplored  and leaves open many directions for future work. In particular, punishment protocols do not seem possible for general heterogeneous settings and the existence of robust algorithms for any heterogeneous setting remains open. Also, stronger notions of equilibrium can be considered such as perfect subgame equilibrium.

% Acknowledgments---Will not appear in anonymized version
\acks{Vianney Perchet acknowledge the support of the French National Research Agency project BOLD (ANR19-CE23-0026-04). This work was also supported in part by a public grant as part of the Investissement d'avenir project, reference ANR-11-LABX-0056-LMH, LabEx LMH, in a joint call with Gaspard Monge Program for optimization, operations research and their interactions with data sciences.}

\addcontentsline{toc}{section}{References}
\bibliography{bibliography}

\newpage
\appendix
\section{Supplementary material for Section~\ref{sec:algo1}}
\label{app:algo1}

This section provides a complete description of \algoone and the proofs of Theorems~\ref{thm:algoone} and \ref{thm:robust1}.

\subsection{Thorough description of \algoone} \label{app:algo1descr}
In addition to Section~\ref{sec:algo1}, the pseudocodes of \estimateM[,]\getrank and \exploone are given here. The following Protocol~\ref{proto:estimM}  describes the estimation of $M$ using the notations introduced in Section~\ref{sec:estimm}.

\begin{protocol}[h]
    \DontPrintSemicolon
     \KwIn{$\beta, T$}
     $t_m \gets 0$ \;
     \While{$\min_k \card N_k^j(t) < \beta^2 K^2\log(T)$}{
     Pull $k \sim \mathcal{U}(K)$;\quad 
     Update $\card N_k^j(t)$ and $\card C_k^j(t)$ ;\quad
     $t_m \gets t_m+1$}
     $\hM \gets 1 + \round\Big(\frac{\log\left(1-\frac{1}{K}\sum_k \hat{p}_k^j(t_M)\right)}{\log\left( 1- \frac{1}{K}\right)}\Big)$\tcp*{$\round(x)=$ closest integer to $x$}
     \textbf{Return} $\hM, t_m$
\caption{\label{proto:estimM}\estimateM}
\end{protocol}

Since the duration $t_m^j$ of \estimateM for player $j$ is random and differs between players, each player continues sampling uniformly at random until $\frac{\gamma_2}{\gamma_1} t_m^j$, with $\gamma_1=\frac{13}{14}$ and $\gamma_2=\frac{16}{15}$. Thanks to this additional \textit{waiting room}, Lemma~\ref{lemma:estim1} below guarantees that all players are sampling uniformly at random until at least $t_m^j$ for any $j$.

The estimation of $M$ here tightly estimates the probability to collide individually for each arm. This restriction provides an additional $M$ factor in the length of this phase in comparison with \citep{musicalchair}, where the probability to collide is globally estimated. This is however required because of the Statistic Sensing, but if $\eta_k$ was always observed, then the protocol from \citet{musicalchair} would be robust. 

Indeed, if we directly estimated the global probability to collide, the selfish player could pull only the best arm. 
The number of observations of $\eta_k$ is larger on this arm, and the estimated probability to collide would thus be positively biased because of the selfish player.

%Let $p_k^j \coloneqq \mP[t \in C_k^j | t \in N_k^j]$ be the collision probability on arm $k$.
%As the sampling strategy of the selfish player is unknown, there is no direct relation between $p_k^j$ and $M$. However, their mean allows to estimate $M$ as $\sum_{k=1}^K p_k^j/K = 1 - (1-\sfrac{1}{K})^{M-1}$.  When the collision indicator is always observed, the term $\frac{\sum_{k} \card C_k^j}{\sum_{k} \card N_k^j}$ directly estimates $\sum_k p_k^j/K$, while in the Statistic sensing setting, $\card N_k^j$ depends on $\mP(X_k>0)$.
%This ratio would then estimate a barycenter of the $p_k^j$ with different weights. This is why players here need to estimate all ratios  $\sfrac{\card C_k^j }{\card N_k^j}$ tightly before taking their mean.
%
\medskip

Afterwards, ranks in~$[M]$ are  attributed to players by sampling uniformly at random in~$[M]$ until observing no collision, as described in Protocol~\ref{proto:getrank}. For the same reason, a waiting room is added to guarantee that all players end this protocol with different ranks.

\begin{protocol}[h] 
    \DontPrintSemicolon
     \KwIn{$\hM, t_m^j, \beta, T$}
     $n \gets \beta^2 K^2\log(T)$ and $j \gets -1$\;
     \For{$t_m^j \log(T)/(\gamma_1 n)$ rounds}{
     \uIf{$j=-1$}{
     Pull $k \sim \mathcal{U}(\hM)$;\quad
     \lIf(\tcp*[f]{no collision}){$r_k(t) > 0$}{$j\gets k$} 
     }
     \lElse{Pull $j$}}
     \textbf{Return} $j$
    \caption{\label{proto:getrank}\getrank}
\end{protocol}

The following quantities are used to describe \exploone in Algorithm~\ref{algo:explo1}:
\begin{itemize}
\item $\cM^j(t) = \left( l_1^j(t), \ldots, l_M^j(t) \right)$ is the list of the empirical $M$ best arms for player~$j$ at round~$t$. It is updated only each $M$ rounds  and ordered according to the index of the arms, \ie $l_1^j(t) < \ldots < l_M^j(t)$.
\item $\hm^j(t)$ is the empirical $M$-th best arm for player $j$ at round $t$.
\item $b_k^j(t) = \sup \lbrace q \geq 0 \ | \ T_k^j(t) \textrm{kl}( \hmu_k^j(t), q) \leq f(t) \rbrace$ is the kl-UCB index of the arm $k$ for player $j$ at round $t$, where $f(t)=\log(t) + 4 \log(\log(t))$, $T_k^j(t)$ is the number of times player $j$ pulled $k$ and $\hmu_k^j$ is the empirical mean. 
%todo: define kl ?
\end{itemize}

\begin{algorithm2e}[h] 
    \DontPrintSemicolon
    \KwIn{$M$, $j$}
    \lIf{$t=0 \ (\mathrm{mod} \ M)$}{Update $\hat{\mu}^j(t), b^j(t), \hm^j(t)$ and $\cM^j(t)= (l_1, \ldots, l_M)$}
    $\pi \gets t+j \ (\mathrm{mod} \ M)+1$ \;
    \lIf{$l_\pi \neq \hm^j(t)$}{Pull $l_\pi$ \tcp*[f]{exploit the $M-1$ best empirical arms}}
    \Else{$\cB^j(t) = \lbrace k \not\in \cM^j(t) \ | \ b_k^j(t) \geq \hmu^j_{\hm^j(t)}(t) \rbrace$ \tcp*{arms to explore}
    \lIf{$\cB^j(t) = \emptyset$}{Pull $l_\pi$}
    \lElse{Pull $\begin{cases} l_\pi \text{ with proba } 1/2 \\ k \text{ chosen uniformly at random in } \cB^j(t) \text{ otherwise} \quad \tcp*[f]{explore}\end{cases}$ }}
    \caption{\label{algo:explo1}\exploone}
\end{algorithm2e}

\subsection{Proofs of Section~\ref{sec:algo1}}

Let us define $\alpha_k \coloneqq \mP(X_k(t) > 0) \geq \mu_k$, $\gamma_1 = \frac{13}{14}$ and $\gamma_2 = \frac{16}{15}$.
\subsubsection{Regret analysis}
\label{app:regretalgo1}
This section aims at proving Theorem~\ref{thm:algoone}. This proof is divided in several auxiliary lemmas given below. First, the regret can be decomposed as follows:

\begin{equation}
\label{eq:regdec1}
R_T = R^{\text{init}} + R^{\text{explo}},
\end{equation}
\begin{equation*}
\text{where } \left\{ \begin{split} \begin{aligned} & R^{\text{init}} = T_0 {\mathlarger\sum_{k = 1}^M} \mu_{(k)} - \mathbb{E}_\mu \Big[{\mathlarger\sum_{t=1}^{T_0}} {\mathlarger \sum_{j = 1}^M} r^j(t) \Big] 
\text{ with } T_{0} = \left(\frac{\gamma_2}{\gamma_1^2 \beta^2 K^2} + \frac{\gamma_2^2}{\gamma_1^2}\right)\ \max_j t_m^j, \\
& R^{\text{explo}} = (T-T_{0}) {\mathlarger\sum_{k = 1}^M} \mu_{(k)} - \mathbb{E}_\mu \Big[{\mathlarger\sum_{t=T_{0}+1}^{T}} {\mathlarger \sum_{j = 1}^M} r^j(t) \Big].  \end{aligned} \end{split} \right.
\end{equation*}

Lemma~\ref{lemma:estim1} first gives guarantees on the \estimateM protocol. Its proof is given in Appendix~\ref{app:proofestim1}.

\begin{lemm}
\label{lemma:estim1}
If $M-1$ players run \estimateM with $\beta\geq 39$, followed by a waiting room until~$\frac{\gamma_2}{\gamma_1} t_m^j$, then regardless of the strategy of the remaining player, with probability larger than~$1-\frac{6KM}{T}$, for any player:
\begin{equation*}
{\hM}^j = M \text{ and } \frac{t_m^j \alpha_{(K)}}{K} \in [\gamma_1 n, \gamma_2 n],
\end{equation*}
where $n = \beta^2 K^2 \log(T)$.
\end{lemm}

When $\hM^j = M$ and $\frac{t_m^j \alpha_{(K)}}{K} \in [\gamma_1 n, \gamma_2 n]$ for any cooperative player $j$, we say that the estimation phase is \textbf{successful}.
\begin{lemm}
\label{lemma:getrank}
Conditioned on the success of the estimation phase, with probability $1-\frac{M}{T}$, all the cooperative players end \getrank with different ranks $j \in [M]$, regardless of the behavior of other players.
\end{lemm}

The proof of Lemma~\ref{lemma:getrank} is given in Appendix~\ref{app:getrankproof}. If the estimation is successful and all players end \getrank with different ranks $j \in [M]$, the initialization is said successful. 

Using the same arguments as \citet{proutiere2019}, the collective regret of the \exploone phase can be shown to be $M \sum_{k>M} \frac{\mu_{(M)} - \mu_{(k)}}{\mathrm{kl}(\mu_{(M)}, \mu_{(k)})} \log(T) + \smallO(\log(T))$. This result is given by Lemma~\ref{lemma:ucb1}, whose proof is given in Appendix~\ref{app:ucb1proof}.

\begin{lemm}
\label{lemma:ucb1}
%Assume $\delta_0 \coloneqq \min_k \frac{\mu_{(k)} - \mu_{(k+1)}}{2} > 0$. 
If all players follow \algoone[:]
\begin{small}
\begin{equation*}
\mE[R^{\text{explo}}] \leq M \sum_{k>M} \frac{\mu_{(M)} - \mu_{(k)}}{\mathrm{kl}(\mu_{(M)}, \mu_{(k)})} \log(T) + \smallO(\log(T)).
\end{equation*}
\end{small}
\end{lemm}

\begin{proof}[Proof of Theorem~\ref{thm:algoone}.]
Thanks to Lemma~\ref{lemma:ucb1}, the total regret is bounded by $$M \sum_{k>M} \frac{\mu_{(M)} - \mu_{(k)}}{\mathrm{kl}(\mu_{(M)}, \mu_{(k)})} \log(T) + \mE[T_0] M + \smallO(\log(T)).$$

Thanks to Lemmas~\ref{lemma:estim1} and \ref{lemma:getrank}, $\mE[T_0] = \cO\left(\frac{K^3 \log(T)}{\mu_{(K)}}\right)$, yielding Theorem~\ref{thm:algoone}.
\end{proof}
\subsubsection{Proof of Lemma~\ref{lemma:estim1}}
\label{app:proofestim1}
\begin{proof}[]
Let $j$ be a cooperative player and $q_k(t)$ be the probability at round $t$ that the remaining player pulls~$k$. Define $p_k^j(t) = \mP[t \in C_k^j(t) \ | \ t \in N_k^j(t)]$. By definition, $p_k^j(t) = 1 - (1-1/K)^{M-2}(1-q_k(t))$ when all cooperative players are pulling uniformly at random. Two auxiliary Lemmas using classical concentration inequalities are used to prove Lemma~\ref{lemma:estim1}. The proofs of Lemmas~\ref{lemma:chernoff1}~and~\ref{lemma:concentration1} are given in Appendix~\ref{app:auxlemmas}.

%%% improved version ? %%%
%\begin{lemm}
%\label{lemma:chernoff1}
%For any $\delta>0$, $T_M$ and fixed sets $(N_k^j(T_M))_{k \in [K]}$,
%\begin{enumerate}
%\item $\mP \left[\bigg|\sum_{k=1}^K\frac{\min_k \card(N_k^j(T_M))}{\card(N_k^j(T_M))}\card(C_k^j(T_M)) - \frac{\min_k \card(N_k^j(T_M))}{\card(N_k^j(T_M))} \sum_{t \in N_k^j(T_M)} p_k^j(t)\bigg| \geq \delta K \min_k \card(N_k^j(T_M)) \right] \leq 2 \exp(-\frac{K\min_k \card(N_k^j(T_M)) \delta^2}{2})$.
%\end{enumerate}
%For $0 < \delta < 1$ and a fixed $T_M$,
%\begin{enumerate}
%\setcounter{enumi}{1}
%\item $\mP \left[\bigg|\card(N_k^j) - \frac{\alpha_k T_M}{K}\bigg| \geq \delta \frac{ \alpha_k T_M}{K} \right] \leq 2\exp(-\frac{T_M \alpha_k \delta^2}{3 K})$.
%\item $\mP \left[\bigg| \sum_{t=1}^{T_M} (\one(t \in N_k^j) - \frac{\alpha_k}{K})p_k^j(t) \bigg| \geq \delta \frac{ \alpha_k T_M}{K} \right] \leq 2 \exp\left( -\frac{T_M \alpha_k \delta^2}{3K}\right)$.
%\end{enumerate}
%\end{lemm}

%%% old version %%%
\begin{lemm}
\label{lemma:chernoff1}
For any $\delta>0$,
\begin{enumerate}
\item $\mP \left[\bigg|\frac{\card C_k^j(T_M)}{\card N_k^j(T_M)} - \frac{1}{\card N_k^j(T_M)} \sum_{t \in N_k^j(T_M)} p_k^j(t)\bigg| \geq \delta \ \Big| \ N_k^j(T_M)\right] \leq 2 \exp(-\frac{\card N_k^j(T_M) \delta^2}{2})$.
\end{enumerate}
For any $\delta \in (0,1)$ and fixed $T_M$,
\begin{enumerate}
\setcounter{enumi}{1}
\item $\mP \left[\bigg|\card N_k^j - \frac{\alpha_k T_M}{K}\bigg| \geq \delta \frac{ \alpha_k T_M}{K} \right] \leq 2\exp(-\frac{T_M \alpha_k \delta^2}{3 K})$.
\item $\mP \left[\bigg| \sum_{t=1}^{T_M} (\one(t \in N_k^j) - \frac{\alpha_k}{K})p_k^j(t) \bigg| \geq \delta \frac{ \alpha_k T_M}{K} \right] \leq 2 \exp\left( -\frac{T_M \alpha_k \delta^2}{3K}\right)$.
\end{enumerate}
\end{lemm}

\begin{lemm}
\label{lemma:concentration1}
For any $k$, $j$ and $\delta \in (0, \frac{\alpha_k}{K})$, with probability larger than~$1-\frac{6KM}{T}$,
\begin{equation*}
\bigg| \hat{p}_k^j(t_m^j) - \frac{1}{t_m^j} \sum_{t=1}^{t_m^j} p_k^j(t) \bigg| \leq 2\sqrt{\frac{6\log(T)}{n \left(1-2\sqrt{\frac{3}{2\beta^2}(1+\frac{3}{2\beta^2}})\right)}} + 2 \sqrt{\frac{\log(T)}{n}}.
\end{equation*}
And for $\beta \geq 39$:
\begin{equation*}
\frac{t_m^j \alpha_{(k)}}{K} \in \left[\frac{13}{14}n, \ \frac{16}{15}n\right].
\end{equation*}
\end{lemm}

Let $\varepsilon = 2\sqrt{\frac{6\log(T)}{n \left(1-2\sqrt{\frac{3}{2\beta^2}(1+\frac{3}{2\beta^2}})\right)}} + 2 \sqrt{\frac{\log(T)}{n}}$ and $p_k^j = \frac{1}{t_m^j} \sum_{t=1}^{t_m^j} p_k^j(t)$ such that with probability at least $1-\frac{6KM}{T}$, $\big| \hat{p}_k^j - p_k^j \big| \leq \varepsilon$. The remaining of the proof is conditioned on this event. \\

By definition of $n$, $\varepsilon = \frac{1}{K} f(\beta)$ where $f(x)=\frac{2}{x}\sqrt{\frac{6}{1-2\sqrt{\frac{3}{2x^2}(1+\frac{3}{2x^2}})}} + 2/x$. Note that $f(x) \leq \frac{1}{2e}$ for $x\geq 39$ and thus $\varepsilon \leq \frac{1}{2Ke}$ for the considered $\beta$.\\

The last point of Lemma~\ref{lemma:concentration1} yields that $t_m^j \leq \frac{\gamma_2}{\gamma_1} t_m^{j'}$ for any pair $j, j'$. All the cooperative players are thus pulling uniformly at random until at least $t_m^j$, thanks to the additional waiting room. Then,
\begin{equation*}
\frac{1}{K} \sum_k (1-p_k^j(t)) \ = \ (1-1/K)^{M-2} (1 - \frac{1}{K} \sum_k q_k(t))
\ = \ (1-1/K)^{M-1}.
\end{equation*}

%This is the important point of our estimation algorithm and its difference with the one from \citet{musicalchair}. Because $q_k(t)$ can be non-uniform, the estimators $\hat{p}_k^j$ can be biased compared to $1-(1-1/K)^{M-1}$. However, when averaging over all arms, this can not be biased by a single adversarial player. The previous protocol \citep{musicalchair} does not estimate $\hat{p}_k^j$ for each individual arm, but instead estimates the probability to observe a collision regardless of the arm choice. With collision or full sensing, it is also robust to an adversarial player. However with statistic sensing, collisions are observed more often on the arms with larger $\alpha_k$. An adversarial player could pull these arms to make other players overestimate $M$. This is why every $p_k^j$ has to be tightly estimated. As a drawback, the required time to estimate $M$ includes an additional $K$ factor in comparison with the previous protocol \citep{musicalchair}. 
%
%\medskip

When summing over $k$, it follows:
\begin{align*}
\frac{1}{K}\sum_k (1-p_k^j) - \varepsilon & \leq \frac{1}{K}\sum_k (1-\hat{p}_k^j) & \leq \frac{1}{K}\sum_k (1-p_k^j) + \varepsilon \\
(1-1/K)^{M-1} - \varepsilon & \leq \frac{1}{K}\sum_k (1-\hat{p}_k^j)  & \leq (1-1/K)^{M-1} + \varepsilon \\
M-1 + \frac{\log(1+ \frac{\varepsilon}{(1-1/K)^{M-1}})}{\log(1-1/K)} & \leq \frac{\log\left(\frac{1}{K}\sum_k (1-\hat{p}_k^j)\right)}{\log(1-1/K)} & \leq M-1 + \frac{\log(1- \frac{\varepsilon}{(1-1/K)^{M-1}})}{\log(1-1/K)} \\
M-1 + \frac{\log(1+ \frac{1}{2K})}{\log(1-1/K)} & \leq \frac{\log\left(\frac{1}{K}\sum_k (1-\hat{p}_k^j)\right)}{\log(1-1/K)} & \leq M-1 + \frac{\log(1- \frac{1}{2K})}{\log(1-1/K)}
\end{align*}

The last line is obtained by observing that $\frac{\varepsilon}{(1-1/K)^{M-1}}$ is smaller than $\frac{1}{2K}$.

Observing that $\max\left(\frac{\log(1-x/2)}{\log(1-x)}, -\frac{\log(1+x/2)}{\log(1-x)}\right) < 1/2$ for any $x>0$, the last line implies:
\begin{equation*}
1+ \frac{\log\Big(\frac{1}{K}\sum_k (1-\hat{p}_k^j)\Big)}{\log(1-1/K)} \in (M-1/2, M+1/2).
\end{equation*}

When rounding this quantity to the closest integer, we thus obtain $M$, which yields the first part of Lemma~\ref{lemma:estim1}. The second part is directly given by Lemma~\ref{lemma:concentration1}.
\end{proof}
\subsubsection{Proof of Lemma~\ref{lemma:getrank}}
\label{app:getrankproof}
The proof of Lemma~\ref{lemma:getrank} relies on two lemmas given below.
\begin{lemm}
\label{lemma:waiting2}
Conditionally on the success of the estimation phase, when a cooperative player~$j$ proceeds to \getrank[,]all other cooperative players are either running \getrank or in a waiting room\footnote{Note that there is a waiting room before \textbf{and} after \getrank[.]}, \ie they are not proceeding to \exploone yet.
\end{lemm}

\begin{proof}
Recall that $\gamma_1 = 13/14$ and $\gamma_2 = 16/15$. Conditionally on the success of the estimation phase, for any pair $(j, j')$, $\frac{\gamma_2}{\gamma_1} t_m^j \geq t_m^{j'}$. Let $t_r^j = \frac{t_m^j}{\gamma_1 K^2 \beta^2}$ be the duration time of \getrank for player $j$. For the same reason, $\frac{\gamma_2}{\gamma_1} t_r^j \geq t_r^{j'}$. Player $j$ ends \getrank at round $t^j = \frac{\gamma_2}{\gamma_1} t_m^j + t_r^j$ and the second waiting room at round $\frac{\gamma_2}{\gamma_1} t^j$. 

As $\frac{\gamma_2}{\gamma_1} t^j \geq t^{j'}$, this yields that when a player ends \getrank[,]all other players are not running \algoone yet. Because $\frac{\gamma_2}{\gamma_1} t_m^j \geq t_m^{j'}$, when a player starts \getrank[,]all other players also have already ended \estimateM[.]This yields Lemma~\ref{lemma:waiting2}.
\end{proof}

\begin{lemm}
\label{lemma:fix1}
Conditionally on the success of the estimation phase, with probability larger than~$1-\frac{1}{T}$, cooperative player $j$ ends \getrank with a rank in $[M]$.
\end{lemm}
\begin{proof}
Conditionally on the success of the estimation phase and thanks to Lemma~\ref{lemma:concentration1}, $t_r^j = \frac{t_m^j}{\gamma_1 K^2 \beta^2} \geq \frac{K \log(T)}{\alpha_{(K)}}$. Moreover, at any round of \getrank[,]the probability of observing $\eta_k(t)=0$ is larger than $\frac{\alpha_{(K)}}{M}$. Indeed, the probability of observing $\eta_k(t)$ is larger than $\alpha_{(K)}$ with Statistic sensing. Independently, the probability of having $\eta_k = 0$ is larger than $1/M$ since there is at least an arm among $[M]$ not pulled by any other player. These two points yield, as $M \leq K$:
\begin{align*}
\mP[\text{player does not observe } \eta_k(t)=0 \text{ for } t_r^j \text{ successive rounds}] & \leq \left(1-\frac{\alpha_{(K)}}{M}\right)^{t_r^j} \\
& \leq \exp\left( -\frac{\alpha_{(K)}t_r^j}{M} \right) \\
& \leq \frac{1}{T}
\end{align*}
Thus, with probability larger than $1-\frac{1}{T}$, player $j$ observes $\eta_k(t)=0$ at least once during \getrank[,]\ie she ends the procedure with a rank in $[M]$.
\end{proof}

\begin{proof}[Proof of Lemma~\ref{lemma:getrank}.]
Combining Lemmas~\ref{lemma:waiting2}~and~\ref{lemma:fix1} yields that the cooperative player $j$ ends \getrank with a rank in $[M]$ and no other cooperative player ends with the same rank. Indeed, when a player gets the rank $j$, any other cooperative player has either no attributed rank (still running \getrank or the first waiting room), or an attributed rank $j'$. In the latter case, thanks to Lemma~\ref{lemma:waiting2}, this other player is either running \getrank or in the second waiting room, meaning she is still pulling $j'$. Since the first player ends with the rank $j$, this means that she did not encounter a collision when pulling $j$ and especially, $j \neq j'$.

\medskip

Considering a union bound among all cooperative players now yields Lemma~\ref{lemma:getrank}.
\end{proof}

\subsubsection{Proof of Lemma~\ref{lemma:ucb1}}
\label{app:ucb1proof}

Let us denote $T_0^j = \left(\frac{\gamma_2}{\gamma_1^2 \beta^2 K^2} + \frac{\gamma_2^2}{\gamma_1^2}\right)\ t_m^j$ such that player $j$ starts running \exploone at time $T_0^j$. This section aims at proving Lemma~\ref{lemma:ucb1}. In this section, the initialization is assumed to be successful. The regret due to an unsuccessful initialization is constant in $T$ and thus $\smallO(\log(T))$. We prove in this section, in case of a successful initialization, the following:
\begin{equation}
\mE[ R^{\text{explo}}] \leq  M \sum_{k>M} \frac{\mu_{(M)} - \mu_{(k)}}{\mathrm{kl}(\mu_{(M)}, \mu_{(k)})} \log(T) +\smallO(\log(T)).
\end{equation}

%\etienne{should we add a fairness theorem? or just a remark}
This proof follows the same scheme as the regret proof from \citet{proutiere2019}, except that there is no leader here. 
Every \textit{bad event} then happens independently for each individual player. This adds a $M$ factor in the regret compared to the follower/leader algorithm\footnote{Which is  not selfish-robust.} used by \citet{proutiere2019}.
For conciseness, we only give the main steps and refer to the original Lemmas in \citep{proutiere2019} for their detailed proof.

We first recall useful concentration Lemmas which correspond to Lemmas~1~and~2 in \citep{proutiere2019}. They are respectively simplified versions of Lemma~5 in \citep{combes2015} and Theorem~10 in \citep{garivier2011}.
\begin{lemm}
\label{lemma:proutiere1}
Let $k \in [K]$, $c>0$ and $H$ be a (random) set such that for all $t$, $\lbrace t \in H \rbrace$ is $\cF_{t-1}$ measurable. Assume that there exists a sequence $(Z_t)_{t\geq 0}$ of binary random variables, independent of all $\cF_t$, such that for $t\in H$, $\pi^j(t)=k$ if $Z_t=1$. Furthermore, if $\mE[Z_t] \geq c$ for any $t$, then:
\begin{equation*}
\sum_{t\geq 1} \mP[t \in H \ | \ |\hmu^j_k(t) - \mu_k| \geq \delta] \leq \frac{4 + 2c/\delta^2}{c^2}.
\end{equation*}
\end{lemm}

\begin{lemm}
\label{lemma:proutiere2}
If player $j$ starts following \exploone at round $T_0^j+1$:
\begin{equation*}
\sum_{t>T_0^j} \mP[b_k^j(t) < \mu_k] \leq 15.
\end{equation*}
\end{lemm}

Let $0 < \delta < \delta_0 \coloneqq \min_k \frac{\mu_{(k)} - \mu_{(k+1)}}{2}$. Besides the definitions given in Appendix~\ref{app:algo1descr}, define the following:
\begin{itemize}
\item $\cM^*$ the list of the $M$-best arms, ordered according to their indices.
\item $\cA^j=\lbrace t > T_0^j \ | \ \cM^j(t) \neq \cM^* \rbrace$.
\item $\cD^j = \lbrace t > T_0^j \ | \ \exists k \in \cM^j(t), \ |\hmu_k^j(t) - \mu_k| \geq \delta \rbrace$.
\item $\cE^j = \lbrace t > T_0^j \ | \ \exists k \in \cM^*, \ b_k^j(t) < \mu_k \rbrace$.
\item $\cG^j = \lbrace t \in \cA^j \setminus \cD^j \ | \ \exists k \in \cM^* \setminus \cM^j(t), \ |\hmu_k^j(t) - \mu_k| \geq \delta \rbrace$.
\end{itemize}

\begin{lemm}
\label{lemma:badevent1}
$
\mE[\card(\cA^j \cup \cD^j)] \leq 8MK^2(6K + \delta^{-2}).
$
\end{lemm}
%todo: refait-on les preuves de Proutiere et Wang ou pas ? Si c'est le cas, ici on n'a pas MK^2 mais M^2K en réalité.
\begin{proof}
Similarly to \citet{proutiere2019}, we have $(\cA^j \cup \cD^j) \subset (\cD^j \cup \cE^j \cup \cG^j)$. We can then individually bound $\mE[\card\cD^j]$, $\mE[\card\cE^j]$ and $\mE[\card\cG^j]$, leading to Lemma~\ref{lemma:badevent1}. The detailed proof is omitted here as it exactly corresponds to Lemmas~3 and~4 in \citep{proutiere2019}.
\end{proof}

\begin{lemm}
\label{lemma:ucb2}
Consider a suboptimal arm $k$ and define $\cH_k^j = \lbrace t \in \{T_0^j +1, \ldots, T \} \setminus (\cA^j \cup \cD^j) \ | \ \pi^j(t)=k\rbrace$. It holds
\begin{equation*}
\mE\left[\card\cH_k^j\right] \leq \frac{\log T + 4 \log(\log T)}{\mathrm{kl}(\mu_k + \delta, \mu_{(M)}-\delta)} + 4 + 2\delta^{-2}.
\end{equation*}
\end{lemm}

Lemma~\ref{lemma:ucb2} can be proved using the arguments of Lemma~5 in \citep{proutiere2019}.

\medskip

\begin{proof}[Proof of Lemma~\ref{lemma:ucb1}.]
If $t \in \cA^j \cup \cD^j$, player $j$ collides with at most one player $j'$ such that $t \not\in \cA^{j'} \cup \cD^{j'}$. 

Otherwise, $t \not\in \cA^j \cup \cD^j$ and player $j$ collides with a player $j'$ only if $t \in \cA^{j'} \cup \cD^{j'}$. Also, she pulls a suboptimal arm $k$ only on an exploration slot, \ie instead of pulling the $M$-th best arm. Thus, the regret caused by pulling a suboptimal arm $k$ when $t \not\in \cA^j \cup \cD^j$ is $(\mu_{(M)} - \mu_k)$ and this actually happens when $t \in \cH_k^j$.

This discussion provides the following inequality, which concludes the proof of Lemma~\ref{lemma:ucb1} when using Lemmas~\ref{lemma:badevent1}~and~\ref{lemma:ucb2} and taking $\delta \to 0$.

\begin{small} \begin{equation*}
 \mE \left[ R^{\text{explo}} \right] \leq \underbrace{2 \sum_{j=1}^M \mE\left[\card(\cA^j \cup \cD^j)\right]}_{\text{collisions}} + \underbrace{\sum_{j\leq M} \sum_{k>M} (\mu_{(M)} - \mu_{(k)})\mE\left[\card\cH_k^j\right]}_{\text{pulls of suboptimal arms}}.
\end{equation*}\end{small}
\end{proof}

\subsubsection{Proof of Theorem~\ref{thm:robust1}}
\label{app:statisticrobust1}
\begin{proof}[\vspace{-1.5em}]
\begin{enumerate}[wide, labelwidth=!, labelindent=0pt]
\item Let us first prove the Nash equilibrium property. Define $\cE = [T_0] \cup \Big( \bigcup\limits_{j\in[M]} (\cA^j \cup \cD^j) \Big)$ with the definitions of $T_0, \cA^j$ and $\cD^j$ given in Appendix~\ref{app:ucb1proof}. 
Thanks to Lemmas~\ref{lemma:estim1}~and~\ref{lemma:getrank}, regardless of the strategy of a selfish player, all other players successfully end the initialization after a time $T_0$ with probability $1 - \cO(KM/T)$. The remaining of the proof is conditioned on this event.

\medskip

The selfish player earns at most $\mu_{(1)} T_0$ during the initialization.
Note that \exploone never uses collision information, meaning that the behavior of the strategic player during this phase does not change the behaviors of the cooperative players. Thus, the optimal strategy during this phase for the strategic player is to pull the best available arm. Let $j$ be the rank of the strategic player\footnote{If the strategic player has no attributed rank, it is the only non-attributed rank in $[M]$.}. For $t \not\in \cE$, this arm is the $k$-th arm of $\cM^*$ with $k = t+j \ (\text{mod } M) +1$. In a whole block of length $M$ in $[T]\setminus \cE$, the selfish player then earns at most $\sum_{k=1}^M \mu_{(k)}$.

Over all, when a strategic player deviates from \exploone[,]she earns at most:
\begin{small}
\begin{equation*}
\mathbb{E}[\rew^j_T(s', s_{-j})] \leq \mu_{(1)} (\card\cE+M) + \frac{T}{M} \sum_{k=1}^M \mu_{(k)}.
\end{equation*}
\end{small}
Note that we here add a factor $\mu_{(1)}$ in the initialization regret. This is only because the true loss of colliding is not $1$ but $\mu_{(1)}$. Also, the additional $\mu_{(1)}M$ term is due to the fact that the last block of length $M$ of \exploone is not totally completed.

Thanks to Theorem~\ref{thm:algoone}, it also comes:
\begin{small}
\begin{equation*}
\mathbb{E}[\rew^j_T(s)] \geq \frac{T}{M} \sum_{k=1}^M \mu_{(k)} -\sum_{k>M} \frac{\mu_{(M)} - \mu_{(k)}}{\mathrm{kl}(\mu_{(k)}, \mu_{(M)})}\log(T)  - \cO \left( \mu_{(1)} \frac{K^3}{\mu_{(K)}} \log(T) \right).
\end{equation*}
\end{small}

Lemmas~\ref{lemma:getrank} and \ref{lemma:badevent1} yield that $\mE[\card\cE] = \cO\left(\frac{K^3 \log(T)}{\mu_{(K)}} \right)$, which concludes the proof.

\item We now prove the $(\alpha, \varepsilon)$-stability of \algoone[.]Let $\varepsilon' = \mE[\cE]+M$. Consider that player $j$ is playing a deviation strategy $s' \in \strat$ such that for some other player $i$ and $l>0$:
\begin{small}
\begin{equation*}
\mE[\rew^i_T(s', s_{-j})] \leq \mathbb{E}[\rew^i_T(s)] - l - (\varepsilon'+M).
\end{equation*}
\end{small}
We will first compare the reward of player $j$ with her optimal possible reward.
The only way for the selfish player to influence the sampling strategy of another player is in modifying the rank attributed to this other player. The total rewards of cooperative players with ranks $j$ and $j'$ only differ by at most $\varepsilon' + M$ in expectation, without considering the loss due to collisions with the selfish player. 
%The $M$ term is only due to the fact that the last block (of length $M$) of \exploone is not totally completed.

The only other way to cause regret to another player $i$ is then to pull $\pi^i(t)$ at time $t$. This incurs a loss at most $\mu_{(1)}$ for player $i$, while this incurs a loss at least $\mu_{(M)}$ for player $j$, in comparison with her optimal strategy. This means that for incurring the additional loss $l$ to the player $i$, player $j$ must suffer herself from a loss $\frac{\mu_{(M)}}{\mu_{(1)}}$ compared to her optimal strategy $s^*$.
Thus, for $\alpha = \frac{\mu_{(M)}}{\mu_{(1)}}$:
\begin{equation*}
\mE[\rew^i_T(s', s_{-j})] \leq \mathbb{E}[\rew^i_T(s)] - l - (\varepsilon'+M) \implies 
\mathbb{E}[\rew^j_T(s', s_{-j})] \leq \mathbb{E}[\rew^j_T(s^*, s_{-j})]  - \alpha l 
\end{equation*}

The first point of Theorem~\ref{thm:robust1} yields for its given $\varepsilon$:
$
\mathbb{E}[\rew^j_T(s^*, s_{-j})] \leq   \mathbb{E}[\rew^j_T(s)] + \varepsilon$.

\medskip

Noting $l_1=l + \varepsilon' + M$ and $\varepsilon_1 = \varepsilon + \alpha(\varepsilon' + M) = \cO(\varepsilon)$, we have shown:
\begin{small}
\begin{equation*}
 \mathbb{E}[\rew^i_T(s', s_{-j})] \leq \mathbb{E}[\rew^i_T(s)] - l_1 \implies \mathbb{E}[\rew^j_T(s', s_{-j})] \leq \mathbb{E}[\rew^j_T(s)] + \varepsilon_1 - \alpha l_1 .
\end{equation*}
\end{small}
\end{enumerate}
\end{proof}

\subsubsection{Auxiliary lemmas}
\label{app:auxlemmas}
This section provides useful Lemmas for the proof of Lemma~\ref{lemma:estim1}. We first recall a useful version of Chernoff bound.
\begin{lemm}
\label{lemma:chernoff0}
For any independent variables $X_1, \ldots, X_n$ in $[0,1]$ and $\delta \in (0,1)$:
\begin{equation*}
\mP\left( \bigg|\sum_{i=1}^n X_i - \mE[X_i]\bigg| \geq \delta\sum_{i=1}^n \mE[X_i] \right) \leq 2 e^{-\frac{\delta^2 \sum_{i=1}^n \mE[X_i]}{3}}.
\end{equation*}
\end{lemm}
%\begin{proof}
%By convexity, $e^{\lambda z} \leq 1 + z (e^\lambda -1)$. $n_k^j(t)$ being in $[0,1]$ of mean $\frac{\alpha_{k}}{K}$, it yields 
%\begin{align*}
%\mE[e^{\lambda(n_k^j(t) - \frac{\alpha_{k}}{K})}] & \leq e^{- \lambda \frac{\alpha_{k}}{K}}(1 + \frac{\alpha_{k}}{K}(e^{\lambda}-1)) \leq e^{- \lambda \frac{\alpha_{k}}{K}}e^{\frac{\alpha_{k}}{K}(e^{\lambda}-1)} \hfill \text{ as } 1+x \leq e^x
%\end{align*}
%The Chernoff bound becomes $\mP[\card(N_k^j) - \frac{T_M \alpha_k}{K} \geq T_M \delta] \leq \min_{\lambda > 0} e^{-\lambda \delta T_M}e^{T_M \alpha_{k} \frac{e^\lambda - \lambda - 1}{K}}$. Especially, this gives for $\lambda = \log(1+\frac{K}{\alpha_{k}}\delta)$:
%\begin{equation*}
%\mP[\card(N_k^j) - \frac{T_M \alpha_k}{K} \geq T_M \delta] \leq \exp\left( -\frac{T_M \alpha_k}{K} h(\frac{K}{\alpha_{k}}\delta) \right) \text{\hfill with } h(u)=(1+u)\log(1+u)-u
%\end{equation*}
%With the same kind of arguments:
%\begin{equation*}
%\mP[\card(N_k^j) - \frac{T_M \alpha_k}{K} \leq -T_M\delta] \leq \exp\left( -\frac{T_M \alpha_k}{K} h(-\frac{K}{\alpha_{k}}\delta) \right)
%\end{equation*}
%As $\delta \leq \alpha_k/K$ and $h(u) \geq \frac{3 u^2}{8}$ for $|u|\leq 1$, we finally have the desired bound:
%\begin{equation*}
%\mP[|\card(N_k^j) - \frac{T_M \alpha_k}{K}| \geq T_M \delta] \leq 2\exp(-\frac{3 T_M K\delta^2}{8 \alpha_k }).
%\end{equation*}
%\end{proof}
\begin{proof}[Proof of Lemma~\ref{lemma:chernoff1}.]
\begin{enumerate}[wide, labelwidth=!, labelindent=0pt]
\item This is an application of Azuma-Hoeffding inequality on the variables ${\one(t \in C_k^j(T_M) ) \ | \ t \in N_k^j(T_M)}$.

\item This is a consequence of Lemma~\ref{lemma:chernoff0} on the variables $\one(t \in N_k^j)$.

\item This is the same result on the variables $\one(t \in N_k^j) p_k^j(t) \ | \mathcal{F}_{t-1}$ where $\mathcal{F}_{t-1}$ is the filtration associated to the past events, using $\sum_{t=1}^{T_M} \mE[\one(t \in N_k^j) p_k^j(t) | \mathcal{F}_{t-1}] \leq \frac{T_M \alpha_k}{K}$.
\end{enumerate}
\end{proof}

\begin{proof}[Proof of Lemma~\ref{lemma:concentration1}.]
From Lemma~\ref{lemma:chernoff1}, it comes:
\begin{itemize}
\item $\mP \left[\exists t \leq T, \Big|\hat{p}_k^j(t) - \frac{1}{\card N_k^j} \sum_{t' \in N_k^j} p_k^j(t')\Big| \geq  2\sqrt{\frac{\log(T)}{\card N_k^j}}\right] \leq \frac{2}{T}$,
\item $\mP \left[\exists t \leq T, \Big|\frac{K\card N_k^j}{\alpha_k t} - 1\Big| \geq \sqrt{\frac{6\log(T)K}{\alpha_kt}} \right] \leq \frac{2}{T}, ~\refstepcounter{equation}\hfill(\theequation)\label{eq:chernoff1}$
\item $\mP \left[\exists t \leq T, \Big|\frac{K}{\alpha_k t} \sum_{t'\in N_k^j} p_k^j(t') - \frac{1}{t}\sum_{t'\leq t} p_k^j(t')\Big| \geq \sqrt{\frac{6 \log(T) K}{\alpha_k t}} \right] \leq \frac{2}{T}$.
\end{itemize}

Noting that $\sum_{t'\in N_k^j} p_k^j(t') \leq \card N_k^j$, Equation~\eqref{eq:chernoff1} implies:
\begin{equation*}
\mP \left[\exists t \leq T, \bigg|\frac{K}{\alpha_k t}\sum_{t'\in N_k^j} p_k^j(t') - \frac{1}{\card N_k^j}\sum_{t'\in N_k^j} p_k^j(t')\bigg| \geq \sqrt{\frac{6 \log(T) K}{\alpha_k t}} \right] \leq  \frac{2}{T}.
\end{equation*}

Combining these three inequalities and making the union bound over all the players and arms yield that with probability larger than $1-\frac{6KM}{T}$:
\begin{equation}
\label{eq:concentration1} \bigg|\hat{p}_k^j(t_m^j) - \frac{1}{t_m^j}\sum_{t\leq t_m^j} p_k^j(t) \bigg| \leq 2\sqrt{\frac{6\log(T) K}{\alpha_k t_m^j}} + 2 \sqrt{\frac{\log(T)}{\card N_k^j(t_m^j)}}.
\end{equation}

Moreover, under the same event, Equation~\eqref{eq:chernoff1} also gives that 
$$N_k^j(t_m^j) \in \bigg[\frac{\alpha_{k} t_m^j}{K} - \sqrt{\frac{6\alpha_k t_m^j\log(T)}{K}}, \ \frac{\alpha_{k} t_m^j}{K} + \sqrt{\frac{6\alpha_k t_m^j\log(T)}{K}}\bigg].$$

Specifically, this yields $n \leq \frac{\alpha_{k} t_m^j}{K} + \sqrt{\frac{6\alpha_{k} t_m^j\log(T)}{K}}$, or equivalently $\frac{t_m^j \alpha_{k}}{K} \geq n-2\sqrt{\frac{3\log(T)}{2}}\sqrt{n+\frac{3\log(T)}{2}}$. Since $n=\beta^2 K^2 \log(T)$, this becomes $\frac{t_m^j \alpha_{k}}{K} \geq n (1-2\sqrt{\frac{3}{2\beta^2K^2}}\sqrt{1+\frac{3}{2\beta^2K^2}})$ and Equation~\eqref{eq:concentration1} now rewrites into:
\begin{equation*}
\bigg|\hat{p}_k^j(t_m^j) - \frac{1}{t_m^j}\sum_{t\leq t_m^j} p_k^j(t) \bigg| \leq 2\sqrt{\frac{6\log(T)}{n \left(1-2\sqrt{\frac{3}{2\beta^2 K^2}(1+\frac{3}{2\beta^2 K^2}})\right)}} + 2 \sqrt{\frac{\log(T)}{n}}
\end{equation*}

Also, $n \geq \frac{\alpha_{k} t_m^j}{K} - \sqrt{\frac{6\log(T)\alpha_{k} t_m^j}{K}}$ for some $k$, which yields $\frac{t_m^j \alpha_{k}}{K} \leq n(1+\frac{3}{\beta^2 K^2}+2\sqrt{\frac{3}{2\beta^2K^2}}\sqrt{1+\frac{3}{2\beta^2K^2}})$. This relation then also holds for $\frac{t_m^j \alpha_{(K)}}{K}$.
We have therefore proved that:
$$
n \left(1-2\sqrt{\frac{3}{2\beta^2}}\sqrt{1+\frac{3}{2\beta^2}}\right) \leq\frac{t_m^j \alpha_{(k)}}{K} \leq n\left(1+\frac{3}{\beta^2}+2\sqrt{\frac{3}{2\beta^2}}\sqrt{1+\frac{3}{2\beta^2}}\right).
$$
For $\beta \geq 39$, this gives the bound in Lemma~\ref{lemma:concentration1}.
\end{proof}

\section{Collective punishment proof}
\label{app:punish}

Recall that the punishment protocol consists in pulling each arm $k$ with probability at least $p_k^j= \max \Big( 1 - \Big(\gamma\frac{\sum_{l=1}^M \hmu^j_{(l)}}{M \hmu^j_k}\Big)^{\frac{1}{M-1}}, 0\Big)$. Lemma~\ref{lemma:punishment} below guarantees that such a sampling strategy is possible.

\begin{lemm}
\label{lemma:punishment}
For $p_k = \max \Big( 1 - \Big(\frac{\gamma \sum_{l=1}^M \hmu^j_{(l)}}{M \hmu^j_k}\Big)^{\frac{1}{M-1}}, 0\Big)$ with $\gamma=\left(1 - 1/K\right)^{M-1}$:  $\sum_{k=1}^K p_k \leq 1$.
\end{lemm}

\begin{proof}
For ease of notation, define $x_k \coloneqq \hmu_k^j$, $\bar{x}_M \coloneqq \sfrac{\sum_{l=1}^M x_{(l)}}{M}$ and $S \coloneqq \{k \in [K] \ | \ x_k > \gamma \bar{x}_M\}= \{k \in [K] \ | \ p_k > 0\}$. We then get by concavity of $x \mapsto -x^{-\frac{1}{M-1}}$,
\begin{align}
\sum_{k\in S} p_k & = \card S \times \left(1- \left(\gamma \bar{x}_M\right)^{\frac{1}{M-1}} \sum_{k \in S} \frac{(x_k)^{-\frac{1}{M-1}}}{\card S}  \right), \\ 
& \leq \card S \times \left(1- \left(\frac{\gamma\bar{x}_M}{\bar{x}_{S}}\right)^{\frac{1}{M-1}} \right) \qquad \text{ with } \bar{x}_{S} = \frac{1}{\card S}\sum_{k \in S} x_k. \label{eq:punish1}
\end{align}

We distinguish two cases.

First, if $\card S \leq M$, we then get $M \bar{x}_M \geq \card S \bar{x}_{S}$ because $S$ is a subset of the $M$ best empirical arms. The last inequality then becomes
\begin{equation*}
\sum_{k \in S} p_k   \leq \card S \left(1- \left(\gamma\frac{\card S}{M}\right)^{\frac{1}{M-1}} \right) .
\end{equation*}
 Define $g(x)= \frac{\gamma}{M} - x(1-x)^{M-1}$.
For $x \in (0, 1]$:
\begin{align*}
g(x) \geq 0 & \iff  \frac{\gamma}{xM} \geq (1-x)^{M-1},\\
& \iff 1- \left(\frac{\gamma}{xM}\right)^{\frac{1}{M-1}} \leq x, \\
& \iff \frac{1}{x} \left(1- \left(\frac{\gamma}{xM}\right)^{\frac{1}{M-1}}\right) \leq 1.
\end{align*}

Thus, $g(\frac{1}{\card S}) \geq0$ implies $\sum_{k\in S} p_k \leq 1$. We now show that $g$ is indeed non negative on $[0,1]$.
$x (1-x)^{M-1}$ is maximized at $\sfrac{1}{M}$ and is thus smaller than $\frac{1}{M}(1-1/M)^{M-1}$, and using the fact that $\frac{1}{M}(1-1/M)^{M-1} \leq \frac{\gamma}{M}$ for our choice of $\gamma$, we get  the result for the first case.

\medskip

The other case corresponds to $\card S > M$. In this case, the $M$ best empirical arms are all in~$S$ and thus $\bar{x}_M \geq \bar{x}_{S}$. Equation~\eqref{eq:punish1} becomes:
$$\sum_{k \in S} p_k  \leq \card S \left(1- \gamma^{\frac{1}{M-1}} \right) \leq K (1-(1-1/K)) = 1. $$
\end{proof}

\section{Supplementary material for \algotwo}
\label{app:sicmmab}

In this whole section, $M$ is assumed to be at least $3$.

\subsection{Description of the algorithm}
\label{app:sicmmab_descript}

This section provides a complete description of \algotwo[.]The pseudocode of \algotwo is given in Algorithm~\ref{alg:algo2} and relies on several auxiliary protocols, which are described by Protocols~\ref{proto:init}, \ref{proto:meansignal},  \ref{proto:receive}, \ref{proto:send}, \ref{proto:update}, \ref{proto:signalset} and \ref{proto:punishhomo}.

\begin{algorithm2e}[h] 
    \DontPrintSemicolon
    \KwIn{$T, \delta$}
    $M, j \gets \init(T, K)$ and punish $\gets$ False\;
    $\opt \gets \emptyset$, $M_p \gets M$, $[K_p] \gets [K]$ and $p \gets 1$\;
 
   \While{not punish and $\card \opt < M$}{
   \For{$m=0, \ldots, \left\lceil \frac{K_p 2^p}{M_p} \right\rceil -1 $}{
	ArmstoPull $\gets \opt \cup \left\lbrace i \in [K_p] \ \big| \ i-mM_p \ (\text{mod } K_p) \in [M_p] \right\rbrace$ \; % any suboptimal arm is pulled at least M 2^p and at most M(2^p +1)
	\For{$M$ rounds}{
	$k \gets j+t \ (\text{mod } M) +1$ and pull $i$ the $k$-th element of ArmstoPull \; 
	\lIf(\tcp*[f]{$T_i^j$ pulls on $i$ by $j$ this phase}){$T_i^j(p) \leq 2^p$}{Update $\hmu_{i}^j$}
	\lIf(\tcp*[f]{collisionless exploration}){$\eta_i = 1$}{punish $\gets$ True}}
   }
   $(\text{punish}, \opt[,] [K_p], M_p) \gets \meansignal(\hmu^j, j, p, \opt, [K_p], M_p)$ \;
   $p \gets p+1$}
\vspace{0.5em} % punishing
\lIf{punish}{$\punishhomo(p)$}
\Else(\tcp*[f]{exploitation phase}){$k \gets j+ t \ (\text{mod } M) +1$ and pull $i$, the $k$-th arm of OptArms \;
\lIf{$\eta_i = 1$}{punish $\gets$ True}}
   \caption{\label{alg:algo2}\algotwo}
\end{algorithm2e}

\begin{protocol}[h]
    \DontPrintSemicolon
    \KwIn{$T, K$}
    $n_{\text{coll}} \gets 0$ and $j \gets -1$ \;
   \lFor(\tcp*[f]{estim. $M$}){$12eK^2 \log(T)$ rounds}{Pull $k \sim \cU(K)$ and $n_{\text{coll}} \gets n_{\text{coll}} + \eta_k$}
   $\hM \gets 1 + \round\left(\log\left(1 -\frac{n_{\text{coll}}}{12eK^2 \log(T)}\right)/\log\left( 1- \frac{1}{K}\right)\right)$\;
   \For(\tcp*[f]{get rank}){$K \log(T)$ rounds}{
	\uIf{$j = -1$}{Pull $k \sim \cU(\hM)$;\quad
	\lIf{$\eta_k = 0$}{$j \gets k$}}
	\lElse{Pull $j$}   
   }
   \Return $(\hM, j)$
    \caption{\label{proto:init}\init}
\end{protocol}

\begin{protocol}[h] 
    \DontPrintSemicolon
   \KwIn{$\hmu^j, j, p, \opt[,][K_p], M_p$}
   punish $\gets$ False\;
	\For(\tcp*[f]{receive punishment signal}){$K$ rounds}{
	Pull $k = t+j \ (\text{mod } K) +1$;\quad
	  \lIf{$\eta_k=1$}{punish $\gets$ True}	
}
   $\widetilde{\mu}_k^j \gets  \begin{cases} 2^{-p} \left(\lfloor 2^p \hmu_k^j \rfloor +1\right) \text{ with proba } 2^p \hmu_k^j - \lfloor 2^p \hmu_k^j \rfloor \\
   2^{-p} \lfloor 2^p \hmu_k^j \rfloor \text{ otherwise } \end{cases}$ \tcp*{quantization}

   \For(\tcp*[f]{$i$ sends $\widetilde{\mu}_k^i$ to $l$}){$(i,l,k) \in [M] \times \lbrace 1, 2 \rbrace \times [K]$ such that $i \neq l$}{
\uIf(\tcp*[f]{sending player}){$j =i$}{$\send(j,l,p,\widetilde{\mu}^j_k)$ and
$q \gets \receive(j,p)$ \tcp*{back and forth}
\lIf(\tcp*[f]{corrupted message}){$q \neq \widetilde{\mu}^j_k$}{punish $\gets$ True}}  
\lElseIf{$j=l$}{$\widetilde{\mu}^i_k \gets \receive(j, p)$ and $\send(j,i,p,\widetilde{\mu}^i_k)$ }
		\lElse{Pull $j$ \tcp*[f]{waiting for others}}}

      \For(\tcp*[f]{leaders check info match}){$(i,l,m,k) \in \{(1,2), (2,1)\} \times [M] \times [K]$}{
\lIf{$j=i$}{$\send(j,l,p,\widetilde{\mu}^m_k)$}
\uElseIf{$j=l$}{$q \gets \receive(j, p)$;\quad
				\lIf{$q \neq \widetilde{\mu}^m_k$}{punish $\gets$ True \tcp*[f]{info differ}}}
		\lElse{Pull $j$ \tcp*[f]{waiting for leaders}}}
		
	\lIf{$j \in \lbrace 1, 2 \rbrace$}{(Acc, Rej) $\gets \update(\widetilde{\mu}, p, \opt, [K_p], M_p)$
	}
	\lElse(\tcp*[f]{arms to accept/reject}){Acc, Rej $\gets  \emptyset$}

(punish, Acc) $\gets \signalset(\text{Acc}, j, \text{punish})$ \;
(punish, Rej) $\gets \signalset(\text{Rej}, j, \text{punish})$ \;

		\Return $(\text{punish}, \opt\cup \text{Acc}, [K_p]\setminus \left( \text{Acc} \cup \text{Rej}\right), M_p- \card \text{Acc})$
    \caption{\label{proto:meansignal}\meansignal}
\end{protocol}

\begin{figure*}[h]
\begin{minipage}{0.45\textwidth}
   \centering
   \begin{protocol}[H]
   \DontPrintSemicolon
		\KwIn{$j$, $p$}
		$\widetilde{\mu} \gets 0$ \;
		\For{$n = 0, \ldots , p$ }{
		Pull $j$ \;
		\lIf{$\eta_{j} (t) = 1$}{$\widetilde{\mu} \gets \widetilde{\mu} + 2^{-n}$}}
		\Return $\widetilde{\mu}$ \tcp*{sent mean}
		\caption*{\receive}
		\caption{\label{proto:receive}\receive}
		\end{protocol}
        \end{minipage} \hfill   
\begin{minipage}{0.45\textwidth}
\centering
          \begin{protocol}[H]
          \DontPrintSemicolon
        \KwIn{$j$, $l$, $p$, $\widetilde{\mu}$}
		$\mathbf{m} \gets$ dyadic writing of $\widetilde{\mu}$ of length $p+1$, \ie $\widetilde{\mu} = \sum_{n=0}^p m_n 2^{-n}$ \;
		\For{$n = 0, \ldots , p$}{
		\lIf(\tcp*[f]{send $1$}){$m_n = 1$}{Pull $l$}
		\lElse(\tcp*[f]{send $0$}){Pull $j$}}
		\caption*{\receive}
		\caption{\label{proto:send}\send}
		\end{protocol}
   \end{minipage}
   \end{figure*}

\begin{protocol}[h] 
    \DontPrintSemicolon
   \KwIn{$\widetilde{\mu}, p, \opt, [K_p], M_p$}
   Define for all $k$, $i^k \gets \argmax_{j \in [M]} \widetilde{\mu}_k^j$ and $i_k \gets \argmin_{j \in [M]} \widetilde{\mu}_k^j$\;
  $\widetilde{\mu}_k \gets \sum_{j \in [M]\setminus \{i^k, i_k \}} \widetilde{\mu}_k^j$ and $b \gets 4 \sqrt{\frac{\log(T)}{(M-2)2^{p+1}}}$\; 
  $\text{Rej} \gets \text{set of arms } k \text{ verifying } \card\left\{ i \in [K_p] \ | \widetilde{\mu}_i - b \geq \widetilde{\mu}_k + b \right\} \geq M_p$ \;
    $\text{Acc} \gets \text{set of arms } k \text{ verifying } \card\left\{ i \in [K_p] \ | \widetilde{\mu}_k - b \geq \widetilde{\mu}_i + b \right\} \geq K_p - M_p$ \;
   \Return (Acc, Rej)
    \caption{\label{proto:update}\update}
\end{protocol}

\begin{protocol}[h] 
    \DontPrintSemicolon
   \KwIn{$S, j, \text{punish}$}
    length\_S $\gets \card S$ \tcp*{length of $S$ for leaders, $0$ for others}
   
\For(\tcp*[f]{leaders send $\card S$}){$K$ rounds}{
	\lIf{$j \in \lbrace 1, 2 \rbrace$}{Pull length\_S}
	\uElse{Pull $k = t+j \ (\text{mod } K) +1$\;
		  \lIf(\tcp*[f]{receive different info}){$\eta_k = 1$ and length\_S $\neq 0$}{punish $\gets$ True}
		  \lIf{$\eta_k = 1$ and length\_S $= 0$}{length\_S $\gets k$}}
}

\For(\tcp*[f]{send/receive $S$}){$n = 1, \ldots, \mathrm{length\_S}$}{
\For{$K$ rounds}{
\lIf{$j \in \lbrace 1, 2 \rbrace$}{Pull $n$-th arm of $S$}
\uElse{Pull $k = t+j \ (\text{mod } K) +1$;\quad 
	  \lIf{$\eta_k=1$}{Add $k$ to S}}
}}
\lIf(\tcp*[f]{corrupted info}){$\card S \neq \mathrm{length\_S}$}{punish $\gets$ True}
	\Return (punish, $S$)
    \caption{\label{proto:signalset}\signalset}
\end{protocol}

\begin{protocol}[h] 
    \DontPrintSemicolon   
    \KwIn{$p$}
    \uIf{communication phase $p$ starts in less than $M$ rounds}{\lFor(\tcp*[f]{signal punish to everyone}){$M+K$ rounds}{Pull $j$}}
 %   \uElseIf{OptArms $\neq \emptyset$}{\lFor{$M$ timesteps}{Pull the first arm of OptArms}}
    \lElse{\lFor{$M$ rounds}{Pull the first arm of ArmstoPull as defined in Algorithm~\ref{alg:algo2}}}
   	$\gamma \gets \left(1 - 1/K\right)^{M-1}$ and $\delta = \frac{1-\gamma}{1+3\gamma}$;\quad 
   	Set $\hmu_k^j, S_k^j, s_k^j, n_k^j \gets 0$\;
   	\While(\tcp*[f]{estimate $\mu_k$}){$\exists k \in [K], \delta \hmu_k^j < 2 s_k^j(\log(T)/n_k^j)^{1/2} + \frac{14 \log(T)}{3 (n_k^j-1)}$}{Pull $k=t+j \ (\text{mod } K) +1$ \;
   	\uIf{$\delta \hmu_k^j < 2 s_k^j(\log(T)/n_k^j)^{1/2} + \frac{14 \log(T)}{3 (n_k^j-1)}$}{Update $\hmu_k^j \gets \frac{n_k^j}{n_k^j+1}\hmu_k^j + X_k(t)$ and $n_k^j \gets n_k^j+1$ \;
   			 Update $S_k^j \gets S_k^j + X_k^2$ and $s_k^j \gets \sqrt{\frac{S_k^j - (\hmu_k^j)^2}{n_k^j-1}}$   			 }
   	
   	}
   	$p_k \gets \bigg( 1 - \Big(\gamma\frac{\sum_{l=1}^M \hmu^j_{(l)}(t)}{M \hmu^j_k(t)}\Big)^{\frac{1}{M-1}}\bigg)_+$;\quad
   	$\widetilde{p}_k \gets p_k/\sum_{l=1}^K p_l$ \tcp*{renormalize}
   	\lWhile(\tcp*[f]{punish}){$t \leq T$}{Pull $k$ with probability $p_k$}
    \caption{\label{proto:punishhomo}\punishhomo}
\end{protocol}

\paragraph{Initialization phase.} The purpose of the initialization phase is to estimate $M$ and attribute ranks in $[M]$ to all the players. This is done by \init[,]which is given in Protocol~\ref{proto:init}. It simply consists in pulling uniformly at random for a long time to infer $M$ from the probability of collision. Then it proceeds to a Musical Chairs procedure so that each player ends with a different arm in $[M]$, corresponding to her rank.

\paragraph{Exploration phase.} As explained in Section~\ref{sec:algo2}, each arm that still needs to be explored (those in $[K_p]$, with  Algorithm~\ref{alg:algo2} notations) is pulled  at least $M2^p$ times during the $p$-th exploration phase.
Moreover, as soon as an arm is found optimal, it is pulled for each remaining round of the exploration.
The last point is that each arm is pulled the exact same amount of time by any player, in order to ensure fairness of the algorithm, while still avoiding collisions. This is the interest of the ArmstoPull set in Algorithm~\ref{alg:algo2}. 
At each time step, the pulled arms are the optimal ones and $M_p$ arms that still need to be explored. The players proceed to a sliding window over these arms to explore, so that the difference in pulls for two arms in $[K_p]$ is at most $1$ for any player and phase. %Thanks to this, any suboptimal arm is not pulled a too large amount of times.

\paragraph{Communication phase.} The pseudocode for a whole communication phase is given by \meansignal in Protocol~\ref{proto:meansignal}. Players first quantize their empirical means before sending them in $p$ bits to each leader. The protocol to send a message is given by Protocol~\ref{proto:send}, while Protocol~\ref{proto:receive} describes how to receive the message. The messages are sent using back and forth procedures to detect corrupted messages.

After this,  leaders communicate the received statistics to each other, to ensure that no player sent differing ones to them.

They can then determine which arms are optimal/suboptimal using \update given by Protocol~\ref{proto:update}. As explained in Section~\ref{sec:algo2}, it cuts out the extreme estimates and decides based on the $M-2$ remaining ones.

\medskip

Afterwards, the leaders signal to the remaining players the sets of optimal and suboptimal arms as described by Protocol~\ref{proto:signalset}. If the leaders send differing information, it is detected by at least one player.

If the presence of a malicious player is detected at some point of this communication phase, then players signal to each other to trigger the punishment protocol described by Protocol~\ref{proto:punishhomo}.

\paragraph{Exploitation phase.} If no malicious player perturbed the communication,  players end up having detected the $M$ optimal arms. As soon as it is the case, they only pull these $M$ arms in a collisionless way until the end.

\newpage
%\subsection{Proof of Section~\ref{sec:algo2}}
%\label{app:proofsicmmab}
%

\subsection{Regret analysis}
\label{app:proofsicmmabregret}
This section aims at proving the first point of Theorem~\ref{thm:sicmmab1}, using similar techniques as in \citep{boursier2018}. The regret is first divided into three parts:

\begin{equation}
\label{eq:regdec2}
R_T = R^{\text{init}} + R^{\text{comm}} + R^{\text{explo}},
\end{equation}
\begin{equation*}
\text{where } \left\{ \begin{split} \begin{aligned} & R^{\text{init}} = T_{\text{init}} {\mathlarger\sum_{k = 1}^M} \mu_{(k)} - \mathbb{E}_\mu \Big[{\mathlarger\sum_{t=1}^{T_{\text{init}}}} {\mathlarger \sum_{j = 1}^M} r^j(t) \Big]  %\hspace{3.2cm} 
\text{ with } T_{\text{init}} = (12eK^2 + K) \log(T), \\
& R^{\text{comm}} = \mathbb{E}_\mu \Big[{\mathlarger\sum_{t \in \text{Comm}}}{\mathlarger \sum_{j=1}^M}  (\mu_{(j)} - r^j(t)) \Big] \text{ with Comm the set of communication steps,} \\
& R^{\text{explo}} = \mathbb{E}_\mu \Big[{\mathlarger\sum_{t \in \text{Explo}}}{\mathlarger \sum_{j=1}^M} (\mu_{(j)} - r^j(t)) \Big] %\hspace{0.4cm} 
\text{ with Explo} = \{T_{\text{init}} +1, \ldots, T \} \setminus \text{Comm.} \end{aligned} \end{split} \right.
\end{equation*}

A communication step is defined as a round where any player is using the \meansignal protocol.
Lemma~\ref{lemma:initfull} provides guarantees about the initialization phase. When all players correctly estimate $M$ and have different ranks after the protocol \init[,] the initialization phase is said successful.

\begin{lemm}
\label{lemma:initfull}
Independently of the sampling strategy of the selfish player, if all other players follow \init[,]with probability at least $1-\frac{3M}{T}$: $\hM^j = M$ and all cooperative players end with different ranks in $[M]$.
\end{lemm}
\begin{proof}
Let $q_k(t) = \mP[\text{selfish player pulls } k \text{ at time } t]$. Then, for any cooperative player $j$ during the initialization phase:
\begin{align*}
\mP[\text{player }j \text{ observes a collision at time }t] & = \sum_{k=1}^K \frac{1}{K} (1-1/K)^{M-2}(1 - q_k(t)) \\
& = (1-1/K)^{M-2}(1-\frac{\sum_{k=1}^K q_k(t)}{K}) \\
& = (1-1/K)^{M-1}
\end{align*}

Define $p = (1-1/K)^{M-1}$ the probability to collide and $\hat{p}^j = \frac{\sum_{t=1}^{12eK^2 \log(T)}\one_{\eta_{\pi^j(t)}=1}}{12eK^2 \log(T)}$ its estimation by player $j$. The Chernoff bound given by Lemma~\ref{lemma:chernoff0} gives:
\begin{align*}
\mP\left[ \left| \hat{p}^j - p \right| \geq \frac{p}{2K} \right] & \leq 2 e^{-\frac{p \log(T)}{e}} \\
& \leq 2/T
\end{align*}
If $\left| \hat{p}^j - p \right| < \frac{p}{2K}$, using the same reasoning as in the proof of Lemma~\ref{lemma:estim1} leads to $1+\frac{\log(1- \hat{p}^j)}{\log(1-1/K)} \in (M-1/2, M+1/2)$ and then $\hM^j=M$.
With probability  at least $1-2M/T$, all cooperative players correctly estimate $M$.

\medskip

Afterwards, the players sample uniformly in $[M]$ until observing no collision. As at least an arm in $[M]$ is not pulled by any other player, at each time step of this phase, when pulling uniformly at random:
\begin{align*}
\mP[\eta_{\pi^j(t)} = 0] \geq 1/M.
\end{align*}

A player gets a rank as soon as she observes no collision. With probability at least $1-(1-1/M)^{n}$, she thus gets a rank after at most $n$ pulls during this phase. Since this phase lasts $K \log(T)$ pulls, she ends the phase with a rank with probability at least $1-1/T$. Using a union bound finally yields that every player ends with a rank and a correct estimation of $M$. Moreover, these ranks are different between all the players, because a player fixes to the arm $j$ as soon as she gets attributed the rank $j$.
\end{proof}

Lemma~\ref{lemma:explosicmmab1} bounds the exploration regret of \algotwo and is proved in Appendix~\ref{app:proofsicmmabexploregret}. Note that a minimax bound can also be proved as done in \citep{boursier2018}. 
%Such a bound is not considered here for sake of clarity.

\begin{lemm}
\label{lemma:explosicmmab1}
If all players follow \algotwo[,]with probability $1- \cO\left( \frac{KM \log(T)}{T} \right)$,
\begin{small}
\begin{equation*}
R^{\text{explo}} = \cO\left(\sum_{k>M} \frac{\log(T)}{\mu_{(M)} - \mu_{(k)}} \right).
\end{equation*}
\end{small}

\end{lemm}

Lemma~\ref{lemma:commsicmmab1} finally bounds the communication regret.

\begin{lemm}
\label{lemma:commsicmmab1}
If all players follow \algotwo[,]with probability $1- \cO\left(\frac{KM \log(T)}{T} + \frac{M}{T} \right)$:
\begin{small}
\begin{equation*}
R^{\text{comm}} = \cO\left( M^2K \log^2\left( \frac{\log(T)}{(\mu_{(M)} - \mu_{(M+1)})^2} \right) \right).
\end{equation*}
\end{small}

\end{lemm}

\begin{proof}
The proof is conditioned on the success of the initialization phase, which happens with probability $1- \cO\left(\frac{M}{T} \right)$. Proposition~\ref{prop:stopexplor1} given in Appendix~\ref{app:proofsicmmabexploregret} yields that with probability ${1- \cO\left(\frac{KM \log(T)}{T} \right)}$, the number of communication phases is bounded by $N = \cO\left(\log\left(  \frac{\log(T)}{(\mu_{(M)} - \mu_{(M+1)})^2}\right) \right)$. The $p$-th communication phase lasts $8 MK (p+1) + 3K + K\card \text{Acc}(p) + K\card \text{Rej}(p)$, where Acc and Rej respectively are the accepted and rejected arms at the $p$-th phase. Their exact definitions are given in Protocol~\ref{proto:update}. An arm is either accepted or rejected only once, so that $\sum_{p=1}^N \card \text{Acc}(p) + \card \text{Rej}(p)=K$. The total length of Comm is thus bounded by:
\begin{align*}
\card \text{Comm}  & \leq \sum_{p=1}^N 8MK(p+1) + 3K + K\card \text{Acc}(p) + K\card \text{Rej}(p) \\
& \leq 8MK \frac{(N+2)(N+1)}{2} +3KN + K^2
\end{align*}
Which leads to $R^{\text{comm}} = \cO\left( M^2K \log^2\left( \frac{\log(T)}{(\mu_{(M)} - \mu_{(M+1)})^2} \right) \right)$ using the given bound for $N$.
\end{proof}

\begin{proof}[Proof of Theorem~\ref{thm:sicmmab1}.]
Using Lemmas~\ref{lemma:initfull}, \ref{lemma:explosicmmab1}, \ref{lemma:commsicmmab1} and equation~\eqref{eq:regdec2} it comes that with probability  $1- \cO\left(\frac{KM \log(T)}{T} \right)$:
\begin{small}
\begin{equation*}
R_T \leq \cO\left( \mathlarger{\sum_{k>M}} \frac{\log(T)}{\mu_{(M)} - \mu_{(k)}} +  M^2K \log^2\left( \frac{\log(T)}{(\mu_{(M)} - \mu_{(M+1)})^2} \right) + MK^2 \log(T) \right).
\end{equation*}
\end{small}

The regret incurred by the low probability event is $\cO(KM^2 \log(T))$, leading to Theorem~\ref{thm:sicmmab1}.
\end{proof}

\subsubsection{Proof of Lemma~\ref{lemma:explosicmmab1}}
\label{app:proofsicmmabexploregret}
Lemma~\ref{lemma:explosicmmab1} relies on the following concentration inequality.
\begin{lemm}
\label{lemma:concentration2}
Conditioned on the success of the initialization and independently of the means sent by the selfish player, if all other players play cooperatively and send uncorrupted messages, for any $k \in [K]$:
\begin{equation*}
\mP[\exists p \leq n, \left| \widetilde{\mu}_k(p) - \mu_k \right| \geq B(p)] \leq \frac{4nM}{T}
\end{equation*}
where $B(p)=4 \sqrt{\frac{\log(T)}{(M-2)2^{p+1}}}$ and $\widetilde{\mu}_k(p)$ is the centralized mean of arm $k$ at the end of phase $p$, once the extremes have been cut out. It exactly corresponds to the $\widetilde{\mu}_k$ of Protocol~\ref{proto:update}. 
\end{lemm}

\begin{proof}
At the end of phase $p$, $(2^{p+1} - 1)$ observations are used for any player $j$ and arm $k$. Hoeffding bound then gives: $\mP\left[\left| \hmu_k^j(p) - \mu_k \right| \geq \sqrt{\frac{\log(T)}{2^{p+1}}}\right] \leq \frac{2}{T}$. The quantization only adds an error of at most $2^{-p}$, yielding for any cooperative player:
\begin{equation}\label{eq:hoeffding1}
\mP\left[\left| \widetilde{\mu}_k^j(p) - \mu_k \right| \geq 2\sqrt{\frac{\log(T)}{2^{p+1}}}\right] \leq \frac{2}{T}
\end{equation}

Assume w.l.o.g. that the selfish player has rank $M$. Hoeffding inequality also yields:
\begin{small}
\begin{equation*}
\mP\left[\bigg| \frac{1}{M-1}\sum_{j=1}^{M-1} \hmu_k^j(p) - \mu_k \bigg| \geq \sqrt{\frac{\log(T)}{(M-1)2^{p+1}}}\right] \leq \frac{2}{T}.
\end{equation*}
\end{small}

Since $\sum_{j=1}^{M-1} 2^p(\widetilde{\mu}_k^j(p) -\hmu_k^j(p))$ is the difference between $M-1$ Bernoulli variables and their expectation, Hoeffding inequality yields
$\mP\left[\left| \frac{1}{M-1}\sum_{j=1}^{M-1} (\widetilde{\mu}_k^j - \hmu_k^j(p)) \right| \geq \sqrt{\frac{\log(T)}{(M-1)2^{p+1}}}\right] \leq \frac{2}{T}$ and:
\begin{equation} \label{eq:hoeffding2}
\mP\left[\left| \frac{1}{M-1} \sum_{j=1}^{M-1} \widetilde{\mu}_k^j(p) - \mu_k \right| \geq 2\sqrt{\frac{\log(T)}{(M-1)2^{p+1}}}\right] \leq \frac{4}{T}.
\end{equation}

Using the triangle inequality combining equations~\eqref{eq:hoeffding1} and \eqref{eq:hoeffding2} yields for any $j \in [M-1]$:
\begin{align}
\mP\left[\Big| \frac{1}{M-2}\sum_{\substack{j' \in [M-1] \\ j' \neq j}} \widetilde{\mu}_k^j(p) - \mu_k \Big| \geq 4 \sqrt{\frac{\log(T)}{(M-2)2^{p+1}}} \right] 
& \leq \mP\Bigg[ \frac{M-1}{M-2} \Big| \frac{1}{M-1}\sum_{j' \in [M-1]} \widetilde{\mu}_k^j(p) - \mu_k \Big| \nonumber \\ & \phantom{\leq} + \frac{1}{M-2} \left| \widetilde{\mu}_k^j(p) - \mu_k  \right|\geq 4 \sqrt{\frac{\log(T)}{(M-2)2^{p+1}}}\Bigg] \nonumber\\
& \leq  \mP\left[\Big| \frac{1}{M-1} \sum_{j=1}^{M-1} \widetilde{\mu}_k^j(p) - \mu_k \Big| \geq 2\sqrt{\frac{\log(T)}{(M-1)2^{p+1}}}\right]\nonumber\\ & \phantom{\leq } + \mP\left[\left| \widetilde{\mu}_k^j(p) - \mu_k \right| \geq 2\sqrt{\frac{\log(T)}{2^{p+1}}}\right] \nonumber\\
& \leq \frac{6}{T}. \label{eq:hoeffding3}
\end{align}

Moreover by construction, no matter what mean sent the selfish player, $$\min_{j \in [M-1]}\frac{1}{M-2}\sum_{\substack{j' \in [M-1] \\ j' \neq j}} \widetilde{\mu}_k^j(p) \leq \widetilde{\mu}_k(p) \leq \max_{j \in [M-1]}\frac{1}{M-2}\sum_{\substack{j' \in [M-1] \\ j' \neq j}} \widetilde{\mu}_k^j(p).$$

Indeed, assume that the selfish player sends a mean larger than any other player. Then her mean as well as the minimal sent mean are cut out and $\widetilde{\mu}_k(p)$ is then equal to the right term. Conversely if she sends the smallest mean, $\widetilde{\mu}_k(p)$ corresponds to the left term. Since $\widetilde{\mu}_k(p)$ is non-decreasing in $\widetilde{\mu}_k^M(p)$, the inequality also holds in the case where the selfish player sends neither the smallest nor the largest mean.

\medskip

Finally, using a union bound over all $j \in [M-1]$ with equation~\eqref{eq:hoeffding3} yields Lemma~\ref{lemma:concentration2}.
\end{proof}

Using classical MAB techniques then yields Proposition~\ref{prop:stopexplor1}.

\begin{prop} \label{prop:stopexplor1}
Independently of the selfish player behavior, as long as the \punishhomo protocol is not used, with probability $1-\cO\left(\frac{KM \log(T)}{T} \right)$, every optimal arm $k$ is accepted after at most $\cO\left(\frac{\log(T)}{(\mu_{k} - \mu_{(M+1)})^2} \right)$ pulls and every sub-optimal arm $k$ is rejected after at most $\cO\left( \frac{\log(T)}{(\mu_{(M)} - \mu_{k})^2} \right)$pulls  during exploration phases.
\end{prop}

\begin{proof}
The fact that the \punishhomo protocol is not started just means that no corrupted message is sent between cooperative players. The proof is conditioned on the success of the initialization phase, which happens with probability  $1-\cO\left(\frac{M}{T} \right)$. Note that there are at most $\log_2(T)$ exploration phases. Thanks to Lemma~\ref{lemma:concentration2}, with probability $1-\cO\left(\frac{KM \log(T)}{T} \right)$, the inequality $ \left| \widetilde{\mu}_k(p) - \mu_k \right| \leq B(p)$ thus holds for any $p$. The remaining of the proof is conditioned on this event. Especially, an optimal arm is never rejected and a suboptimal one never accepted.

\medskip

First consider an optimal arm $k$ and note $\Delta_k = \mu_k - \mu_{(M+1)}$ the optimality gap. Let $p_k$ be the smallest integer $p$ such that $(M-2)2^{p+1} \geq \frac{16^2 \log(T)}{\Delta^2_k}$. In particular, $4 B(p_k) \leq \Delta_k$, which implies that the arm $k$ is accepted at the end of the communication phase $p_k$ or before.

Necessarily, $(M-2)2^{p_k+1} \leq \frac{2\cdot 16^2 \log(T)}{\Delta^2_k}$ and especially, $M 2^{p_k+1} = \cO\left( \frac{\log(T)}{\Delta_k^2} \right)$. Note that the number of exploratory pulls on arm $k$ during the $p$ first phases is bounded by $M (2^{p+1}+p)$\footnote{During the exploration phase $p$, any explored arm is pulled between $M2^p$ and $M(2^p+1)$ times.}, leading to Proposition~\ref{prop:stopexplor1}.
The same holds for the sub-optimal arms with $\Delta_k = \mu_{(M)} - \mu_{k}$.
\end{proof}

In the following, we keep the notation $t_k = \frac{c \log(T)}{\left( \mu_k-\mu_{(M)} \right)^2}$, where $c$ is a universal constant, such that with probability $1 -\cO\left(\frac{KM}{T}\right)$, any arm $k$ is correctly accepted or rejected after a time at most $t_k$. 
All players are now assumed to play \algotwo[,]\eg there is no selfish player. Since there is no collision during exploration/exploitation (conditionally on the success of the initialization phase), the following decomposition holds \citep{anantharam}:
\begin{small}
\begin{equation}\label{eq:exploregdec1}
R^{\text{explo}} = \sum_{k>M} (\mu_{(M)} - \mu_{(k)}) T_{(k)}^{\text{explo}} + \sum_{k \leq M} (\mu_{(k)} - \mu_{(M)}) (T^{\text{explo}} - T_{(k)}^{\text{explo}}),
\end{equation}
\end{small}where $T^{\text{explo}} = \card\text{Explo}$ and $T_{(k)}^{\text{explo}}$ is the centralized number of pulls on the $k$-th best arm during exploration or exploitation.

\begin{lemm}
\label{lemma:explosicmmab2}
If all players follow \algotwo[,]with probability $1-\cO\left(\frac{KM \log(T)}{T} \right)$, it holds:
\begin{itemize}
\item for $k > M$, $ (\mu_{(M)} - \mu_{(k)}) T_{(k)}^{\text{explo}} = \cO \left( \frac{\log(T)}{\mu_{(M)} - \mu_{(k)}} \right)$.
\item $\sum_{k \leq M} (\mu_{(k)} - \mu_{(M)}) (T^{\text{explo}} - T_{(k)}^{\text{explo}}) = \cO\left(\sum_{k>M}\frac{\log(T)}{\mu_{(M)} - \mu_k} \right)$.
\end{itemize}
\end{lemm}

\begin{proof}
With probability $1-\cO\left(\frac{KM \log(T)}{T} \right)$, Proposition~\ref{prop:stopexplor1} yields that any arm $k$ is correctly accepted or rejected at time at most $t_k$. The remaining of the proof is conditioned on this event and the success of the initialization phase.
The first point of Lemma~\ref{lemma:explosicmmab2} is a direct consequence of Proposition~\ref{prop:stopexplor1}. It remains to prove the second point.

\medskip
Let $\hat{p}_k$ be the number of the phase at which the arm $k$ is either accepted or rejected and let  $K_p$ be the number of arms that still need to be explored at the beginning of phase $p$ and $M_p$ be the number of optimal arms that still need to be explored.
The following two key Lemmas are crucial  to obtain the second point.
\begin{lemm} \label{lemma:sicmmabaux1}
Under the assumptions of  Lemma~\ref{lemma:explosicmmab2}: 
\begin{small}
\begin{equation*}
\sum_{k \leq M} (\mu_{(k)} - \mu_{(M)}) (T^{\text{explo}} - T_{(k)}^{\text{explo}}) \leq \sum_{j > M} \sum_{k \leq M} \sum_{p=1}^{\min (\hat{p}_{(k)}, \hat{p}_{(j)})} (\mu_{(k)} - \mu_{(M)}) 2^p \frac{M}{M_p} + \smallO(\log(T)).
\end{equation*}
\end{small}
\end{lemm}
\begin{lemm} \label{lemma:sicmmabaux2}
Under the assumptions of Lemma~\ref{lemma:explosicmmab2}, for any $j>M$:
\begin{small}
\begin{equation*}
 \sum_{k \leq M} \sum_{p=1}^{\min (\hat{p}_{(k)}, \hat{p}_{(j)})} (\mu_{(k)} - \mu_{(M)}) 2^p \frac{M}{M_p} \leq \cO\left( \frac{\log(T)}{\mu_{(M)} - \mu_{(j)}} \right).
\end{equation*}
\end{small}
\end{lemm}

Combining these two Lemmas with Equation~\eqref{eq:exploregdec1} finally yields Lemma~\ref{lemma:explosicmmab1}.
\end{proof}

\begin{proof}[Proof of Lemma~\ref{lemma:sicmmabaux1}.]
Consider an optimal arm $k$. During the $p$-th exploration phase, either $k$ has already been accepted and is pulled $ M \left\lceil\frac{K_p2^p}{M_p} \right\rceil$ times; or $k$ has not been accepted yet and is pulled at least $2^p M$, \ie is not pulled at most $M \left( \left\lceil \frac{K_p2^p}{M_p} \right\rceil - 2^p \right)  $ times.
This gives:
\begin{align*}
(\mu_{(k)} - \mu_{(M)}) (T^{\text{explo}} - T_{(k)}^{\text{explo}})  & \leq \sum_{p=1}^{\hat{p}_k} (\mu_{(k)} - \mu_{(M)}) M \left(\left\lceil \frac{K_p2^p}{M_p}\right\rceil - 2^p\right), \\
& \leq \sum_{p=1}^{\hat{p}_k} (\mu_{(k)} - \mu_{(M)})  M\left(\frac{K_p 2^p}{M_p}- 2^p +1\right),\\
& \leq \hat{p}_k (\mu_{(k)} - \mu_{(M)}) M + \sum_{p=1}^{\hat{p}_k} (\mu_{(k)} - \mu_{(M)})  (K_p - M_p) \frac{M}{M_p} 2^p.\\
\end{align*}

We assumed that any arm $k$ is correctly accepted or rejected after a time at most $t_k$. This implies that $\hat{p}_k=\smallO(\log(T))$. Moreover, $K_p - M_p$ is the number of suboptimal arms not rejected at phase $p$, \ie $K_p - M_p = \sum_{j>M}\one_{p \leq \hat{p}_{(j)}}$ and this proves Lemma~\ref{lemma:sicmmabaux1}.
\end{proof}

\begin{proof}[Proof of Lemma~\ref{lemma:sicmmabaux2}.]
For $j > M$, define $A_j =  \sum_{k \leq M} \sum_{p=1}^{\min (\hat{p}_{(k)}, \hat{p}_{(j)})} (\mu_{(k)} - \mu_{(M)}) 2^p \frac{M}{M_p}$. We want to show $A_j \leq  \cO\left( \frac{\log(T)}{\mu_{(M)} - \mu_{(j)}} \right)$ with the considered conditions.
Note $T(p) = M(2^{p+1}-1)$ and $\Delta(p) = \sqrt{\frac{c \log(T)}{T(p)}}$. The inequality $\hat{p}_{(k)} \geq p$ then implies $\mu_{(k)} - \mu_{(M)} < \Delta(p)$, \ie
\begin{align*}
A_j & \leq \sum_{k \leq M} \sum_{p=1}^{\hat{p}_{(j)}} 2^p \Delta(p) \one_{p \leq \hat{p}_{(k)}} \frac{M}{M_p} = \sum_{p=1}^{\hat{p}_{(j)}} 2^p \Delta(p) M \\
& \leq \sum_{p=1}^{\hat{p}_{(j)}} \Delta(p) (T(p) - T(p-1))
\end{align*}
The equality comes because $\sum_{k \leq M} \one_{p \leq \hat{p}_{(k)}}$ is exactly $M_p$. Then from  the definition of $\Delta(p)$:
\begin{align*}
A_j & \leq c \log(T) \sum_{p=1}^{\hat{p}_{(j)}} \Delta(p) \left(\frac{1}{\Delta(p)} + \frac{1}{\Delta(p-1)} \right)\left(\frac{1}{\Delta(p)} - \frac{1}{\Delta(p-1)} \right) \\
& \leq (1+\sqrt{2})c \log(T) \sum_{p=1}^{\hat{p}_{(j)}}\left(\frac{1}{\Delta(p)} - \frac{1}{\Delta(p-1)} \right) \\
& \leq (1+\sqrt{2})c \log(T) /\Delta(\hat{p}_{(j)}) \\
& \leq (1+\sqrt{2}) \sqrt{c \log(T) T(\hat{p}_{(j)})}
\end{align*}
%\vspace{-0.1em}
By definition, $T(\hat{p}_{(j)})$ is smaller than the number of exploratory pulls on the $j$-th best arm and is thus bounded by $\frac{c \log(T)}{(\mu_{(M)} - \mu_{(j)})^2}$, leading to Lemma~\ref{lemma:sicmmabaux2}.
\end{proof}

\subsection{Selfish robustness of \algotwo}
\label{app:greedyproofsicmmab}

In this section, the second point of Theorem~\ref{thm:sicmmab1} is proven. First Lemma~\ref{lemma:punishhomo} gives guarantees for the punishment protocol. Its proof is given in Appendix~\ref{app:punishhomo}. 
%We denote by $s$ the collective profile where all players play \algotwo[.]

\begin{lemm} \label{lemma:punishhomo}
If the \punishhomo protocol is started at time $T_{\mathrm{punish}}$ by $M-1$ players, then for the remaining player $j$, independently of her sampling strategy:
\begin{equation*}
\mE[\rew_T^j | \mathrm{punish}] \leq \mE[\rew_{T_{\text{punish}} + t_p}^j] + \widetilde{\alpha} \frac{T - T_{\mathrm{punish}} - t_p}{M} \sum_{k=1}^M \mu_{(k)},
\end{equation*}
with $t_p = \cO\left(\frac{K}{(1-\widetilde{\alpha})^2 \mu_{(K)}}\log(T) \right)$ and $\widetilde{\alpha}=\frac{1+(1-1/K)^{M-1}}{2}$.
\end{lemm}

\begin{proof}[Proof of the second point of Theorem~\ref{thm:sicmmab1} (Nash equilibrium).]
First fix $T_{\text{punish}}$ the time at which the punishment protocol starts if it happens (and $T$ if it does not). Before this time, the selfish player can not perturb the initialization phase, except by changing the ranks distribution. Moreover, the exploration/exploitation phase is not perturbed as well, as claimed by Proposition~\ref{prop:stopexplor1}. The optimal strategy then earns at most $T_{\text{init}}$ during the initialization and $\card\text{Comm}$ during the communication. 
With probability $1-\cO\left(\frac{KM \log(T)}{T} \right)$, the initialization is successful and the concentration bound of Lemma~\ref{lemma:concentration1} holds for any arm and player all the time. The following is conditioned on this event.

\medskip

Note that during the exploration, the cooperative players pull any arm the exact same amount of times. Since the upper bound time $t_k$ to accept or reject an arm does not depend on the strategy of the selfish player, Lemma~\ref{lemma:explosicmmab2} actually holds for the cooperative player, \ie for any cooperative player $j$:
\begin{equation}\label{eq:individualpulls}
\sum_{k\leq M}\left( \mu_{(k)} - \mu_{(M)} \right) \left( \frac{T^{\text{explo}}}{M} - T_{(k)}^j \right) = \cO\left(\frac{1}{M}\sum_{k>M}\frac{\log(T)}{\mu_{(M)} - \mu_k} \right),
\end{equation}
where $T_{(k)}^j$ is the number of pulls by player $j$ on the $k$-th best arm during the exploration/exploitation. The same kind of regret decomposition as in Equation~\eqref{eq:exploregdec1} is possible for the regret of the selfish player $j$ and especially:

\begin{equation*}
R^{\text{explo}}_j \geq \sum_{k \leq M} (\mu_{(k)} - \mu_{(M)}) \left(\frac{T^{\text{explo}}}{M} - T_{(k)}^j\right).
\end{equation*}
However, the optimal strategy for the selfish player is to pull the best available arm during the exploration and especially to avoid collisions. This implies the constraint $T_{(k)}^j \leq T^{\text{explo}} - \sum_{j \neq j'} T_{(k)}^{j'}$. 
Using this constraint with Equation~\eqref{eq:individualpulls} yields $\frac{T^{\text{explo}}}{M} - T_{(k)}^j \geq - \sum_{j\neq j'}\frac{T^{\text{explo}}}{M} - T_{(k)}^{j'}$ and then
$$
R^{\text{explo}}_j \geq -\cO\left( \sum_{k>M}\frac{\log(T)}{\mu_{(M)} - \mu_k}\right),
$$
which can be rewritten as
$$
\rew^{\text{explo}}_j \leq \frac{T^{\text{explo}}}{M} \sum_{k=1}^M \mu_{(k)} + \cO\left( \sum_{k>M}\frac{\log(T)}{\mu_{(M)} - \mu_k}\right).
$$
Thus, for any strategy $s'$ when adding the low probability event of a failed exploration or initialization,
\begin{align*}
 \mE[\rew_{t_p + T_{\text{punish}}}^j(s',s_{-j})]&  \leq (T_{\text{init}} + \card\text{Comm} + t_p + \cO(KM \log(T))) \\ & \phantom{\leq} + \frac{\mE[T_{\text{punish}}] - T_{\text{init}} - \card\text{Comm}}{M} \sum_{k \leq M} \mu_{(k)} + \cO\left( \sum_{k>M}\frac{\log(T)}{\mu_{(M)} - \mu_k}\right).
\end{align*}
Using Lemma~\ref{lemma:punishhomo}, this yields:
\begin{align*}
\mE[\rew_{T}^j(s',s_{-j})] & \leq (T_{\text{init}} + \card\text{Comm} + t_p + \cO(KM \log(T))) \\ & \phantom{\leq} + \frac{\mathbb{E}[T_{\text{punish}}] - T_{\text{init}} - \card\text{Comm}}{M} \sum_{k \leq M} \mu_{(k)} + \cO\left( \sum_{k>M}\frac{\log(T)}{\mu_{(M)} - \mu_k}\right)\\ & \phantom{\leq} + \widetilde{\alpha} \frac{T - \mE[T_{\text{punish}}]}{M} \sum_{k=1}^M \mu_{(k)}.  
\end{align*}
The right term is maximized when $\mE[T_{\text{punish}}]$ is maximized, \ie when it is $T$. We then get:
\begin{small}
\begin{equation*}
\mE[\rew_{T}^j(s',s_{-j})] \leq \frac{T}{M}\sum_{k \leq M} \mu_{(k)} + \varepsilon, 
\end{equation*}
\end{small}
where $\varepsilon= \cO\bigg(\sum_{k>M}\frac{\log(T)}{\mu_{(M)} - \mu_k} + K^2 \log(T) + MK \log^2\left(\frac{\log(T)}{(\mu_{(M)}-\mu_{(M+1)})^2}\right) + \frac{K \log(T)}{(1-\widetilde{\alpha})^2 \mu_{(K)}} \bigg).$
\end{proof}

\begin{proof}[Proof of the second point of Theorem~\ref{thm:sicmmab1} (stability).]
Define $\cE$ the \textit{bad event} that the initialization is not successful or that an arm is poorly estimated at some time. Let $\varepsilon' = T \mP[\cE] + \mE[\card \text{Comm} | \neg \cE] + K \log(T)$. Then $\varepsilon' = \cO \left( KM\log(T) + KM \log^2\left( \frac{\log(T)}{(\mu_{(M)} - \mu_{(M+1)})^2} \right) \right)$.

Assume that the player $j$ is playing a deviation strategy $s'$ such that for some other player $i$ and $l>0$:
\begin{equation*}
\mE[\rew_T^i(s',s_{-j})] \leq \mE[\rew_T^i(s)]-l-\varepsilon'
\end{equation*}

First fix $T_{\text{punish}}$ the time at which the punishment protocol starts. Let us now compare $s'$ with the individual optimal strategy for player $j$, $s^*$. Let $\varepsilon'$ take account of the communication phases, the initialization and the low probability events. 

\medskip

The number of pulls by any player during exploration/exploitation is given by Equation~\eqref{eq:individualpulls} unless the punishment protocol is started. Moreover, the selfish player causes at most a collision during exploration/exploitation before initiating the punishment protocol, so the loss of player $i$ before punishment is at most $1+\varepsilon'$.

\medskip

After $T_{\text{punish}}$, Lemma~\ref{lemma:punishhomo} yields that the selfish player suffers a loss at least $(1-\widetilde{\alpha})\frac{T - T_{\text{punish}} - t_p}{M} \sum_{k=1}^M \mu_{(k)}$, while any cooperative player suffers at most $\frac{T - T_{\text{punish}}}{M} \sum_{k=1}^M \mu_{(k)}$.

The selfish player then suffers after $T_{\text{punish}}$ a loss at least $(1-\widetilde{\alpha})((l-1) - t_p)$.
Define $\beta = 1-\widetilde{\alpha}$. We just showed:
\begin{equation*}
\mE[\rew_T^i(s', s_{-j})] \leq \mE[\rew_T^i(s)]-l-\varepsilon' \implies \mE[\rew_T^j(s', s_{-j})] \leq \mE[\rew_T^j(s^*, s_{-j})]-\beta (l-1) + \beta t_p
\end{equation*}

Moreover, thanks to the second part of Theorem~\ref{thm:sicmmab1}, $\mE[\rew_T^j(s^*, s_{-j})] \leq \mE[\rew_T^j(s)] +\varepsilon$ with $\varepsilon= \cO\bigg(\sum_{k>M}\frac{\log(T)}{\mu_{(M)} - \mu_k} + K^2 \log(T) + MK \log^2\left(\frac{\log(T)}{(\mu_{(M)}-\mu_{(M+1)})^2}\right) + \frac{K \log(T)}{(1-\widetilde{\alpha})^2 \mu_{(K)}}  \bigg) $.
Then by defining $l_1 = l + \varepsilon'$, $\varepsilon_1 = \varepsilon + \beta t_p + \beta \varepsilon' +1 = \cO(\varepsilon)$, we get:
\begin{equation*}
\mE[\rew_T^i(s', s_{-j})] \leq \mE[\rew_T^i(s)]-l_1 \implies \mE[\rew_T^j(s', s_{-j})] \leq \mE[\rew_T^j(s)] + \varepsilon_1 - \beta l_1.
\end{equation*}
\end{proof}

\subsubsection{Proof of Lemma~\ref{lemma:punishhomo}.}
\label{app:punishhomo}
The punishment protocol starts by estimating all means $\mu_k$ with a multiplicative precision of~$\delta$. This is possible thanks to Lemma~\ref{lemma:multiplicativeestim}, which corresponds to Theorem~9 in \citep{cesa2019} and Lemma~13 in \citep{berthet2017}. 

\begin{lemm}\label{lemma:multiplicativeestim}
Let $X_1, \ldots, X_n$ be $n$-i.i.d. random variables in $[0,1]$ with expectation $\mu$ and define $S_t^2 = \frac{1}{t-1}\sum_{s=1}^t(X_s-\bar{X}_t)^2$. For all $\delta \in (0,1)$, if $n \geq n_0$, where
\begin{equation*}
n_0 = \left\lceil \frac{2}{3 \delta \mu} \log(T) \left(\sqrt{9\frac{1}{\delta^2} + 96 \frac{1}{\delta} +85} +\frac{3}{\delta} +1 \right) \right\rceil +2 = \cO\left(\frac{1}{\delta^2 \mu}\log(T) \right)
\end{equation*}
and $\tau$ is the smallest time $t \in \lbrace 2, \ldots, n \rbrace$ such that
\begin{equation*}
\delta \bar{X}_t \geq 2 S_t\left( \log(T)/t \right)^{1/2} + \frac{14 \log(T)}{3(t-1)},
\end{equation*}
then, with probability at least $1-\frac{3}{T}$:
\begin{enumerate}
\item $\tau \leq n_0$,
\item $\left( 1- \delta \right) \bar{X}_\tau < \mu < \left(1+\delta\right) \bar{X}_\tau$.
\end{enumerate}
\end{lemm}

\begin{proof}[Proof of Lemma~\ref{lemma:punishhomo}]

The punishment protocol starts for all cooperative players at $T_{\text{punish}}$. For $\delta = \frac{1-\gamma}{1+3\gamma}$, each player then estimates each arm. Lemma~\ref{lemma:multiplicativeestim} gives that with probability at least $1-3/T$:
\begin{itemize}
\item the estimation ends after a time at most $t_p = \cO\left(\frac{K}{\delta^2 \mu_{(K)}}\log(T) \right)$,
\item $(1-\delta) \hmu_k^j \leq \mu_k \leq (1+\delta) \hmu_k^j$.
\end{itemize}

The following is conditioned on this event. The last inequality can be reversed as $\frac{\mu_k}{1+\delta} \leq \hmu_k^j \leq \frac{\mu_k}{1-\delta}$. Then, this implies for any cooperative player $j$
\begin{align*}
1- p_k^j \leq \left(\gamma \frac{(1+\delta)\sum_{m=1}^M \mu_{(m)}}{(1-\delta)M \mu_k} \right)^{\frac{1}{M-1}}.
\end{align*}

The expected reward that gets the selfish player $j$ by pulling $k$ after the time $T_{\text{punish}}+t_p$ is thus smaller than $\gamma \frac{1+\delta}{1-\delta} \frac{\sum_{m=1}^M \mu_{(m)}}{M} $. 

Note that $\gamma \frac{1+\delta}{1-\delta} = \frac{1+\gamma}{2} = \widetilde{\alpha}$.
Considering the low probability event given by Lemma~\ref{lemma:multiplicativeestim} adds a constant term that can be counted in $t_p$. This finally yields the result of Lemma~\ref{lemma:punishhomo}.
\end{proof}
      
\section{Supplementary material for \algothree}
\label{app:rsd}

\subsection{Description of the algorithm}
\label{app:rsd_descript}

This section provides a complete description of \algothree[.]Its pseudocode is given in Algorithm~\ref{alg:algo3}. It relies on auxiliary protocols described by Protocols~\ref{proto:init}, \ref{proto:rsd}, \ref{proto:listen}, \ref{proto:sendbit}, \ref{proto:prefsignal} and \ref{proto:punishsemi}.

\begin{algorithm2e}[h] 
    \DontPrintSemicolon
    \KwIn{$T, \delta$}
    $\hM, j \gets \init(T, K)$;\ state $\gets$ ``exploring'' and $\text{blocknumber} \gets 1$ \;
    Let $\pmb{\pi}$ be a $M\! \times\! M$ matrix with only $0$ \tcp*{$\pi_k^j$ is the $k$-th preferred arm by $j$}
	
	\While{$t < T$}{
		$\text{blocktime} \gets t \ (\text{mod } 5K+MK+M^2K)+1$ \;
	\If(\tcp*[f]{new block}){blocktime $= 1$}{blocknumber $\gets \text{blocknumber}\ (\text{mod } M) + 1$;\ $b_k^j(t) \gets \sqrt{2\log(T)/T_k^j(t)}$ \;
	Let $\lambda^j$ be the ordering of the empirical means: $\hmu_{\lambda^j_k}^j(t) \geq \hmu_{\lambda^j_{k+1}}^j(t)$ for any $k$\;
	\lIf(\tcp*[f]{send Top-M arms}){(blocknumber, state) $=(j, \text{``exploring''})$ and $\forall k \in [M], \hat{\mu}_{\lambda^j_k}^j - b_{\lambda^j_k}^j \geq \hat{\mu}_{\lambda^j_{k+1}}^j + b_{\lambda^j_{k+1}}^j$\\}{$\pi^j \gets \lambda^j$;\ state $\gets \prefsignal(\pmb{\pi}, j)$}
	}
	$(l, \text{comm\_arm}) \gets \rsd(\pmb{\pi}, \text{blocknumber})$ \tcp*{$j$ pulls $l^j$}
	\vspace{0.5em}%explore
	\If{state $=$ ``exploring''}{ 
	Pull $l^j$ and update $\hmu_{l^j}^j$ \;
	\uIf(\tcp*[f]{received signal}){$l^j = \text{comm\_arm}$ and $\eta_{l^j} = 1$}{  \lIf{blocktime $> 4K$}{state $\gets$ ``punishing''}
	\lElse{$(\text{state}, \pi^{\text{blocknumber}}) \gets \listen(\text{blocknumber}, \text{state}, \pmb{\pi}, \text{comm\_arm})$}}
}

 \vspace{0.5em}% exploit following RSD
\If{state $=$ ``exploiting'' and $\exists i,k \text{ such that } \pi^i_k = 0$}{
Pull $l^j$ \tcp*{arm attributed by RSD algo}

%	\lIf{$l^j \not\in \lbrace l^i | i \in [M]\setminus\{j\} \rbrace$ and $\eta_{l^j} = 1$ \tcp*{received signal}}{ \\ last\_state $\gets$ ``exploiting'';\ state $\gets$ ``punishing''}
\uIf(\tcp*[f]{received signal}){$l^j \not\in \lbrace l^i | i \in [M]\setminus\{j\} \rbrace$ and $\eta_{l^j}(t) = 1$}{  \lIf{blocktime $> 4K$}{state $\gets$ ``punishing''}
	\lElse{ $(\text{state}, \pi^{\text{blocknumber}}) \gets \listen(\text{blocknumber}, \text{state}, \pmb{\pi}, \text{comm\_arm})$}}
%\lIf{$l^j \not\in \lbrace l^i | i \in [M]\setminus\{j\} \rbrace$ and $\eta_{i^j}(t) = 1$\tcp*{received signal}}{last\_state $\gets$ ``exploiting'';\ state  $\gets \begin{cases} \text{``listening'' if blocktime } \leq 4K \\ \text{``punishing'' otherwise} \end{cases}$}
}
\vspace{0.5em} % punishing

%\lIf{state $=$ ``start coll. exploit.''}{
%	$(j, \pmb{\pi}, \text{state}) \gets \redraw(j, \pmb{\pi})$
%}

\If(\tcp*[f]{all players are exploiting}){state $=$ ``exploiting'' and $\forall i,k, \pi^i_k \neq 0$}{
Draw inspect $\sim$ Bernoulli$(\sqrt{\log(T)}/T)$ \;
\uIf(\tcp*[f]{random inspection}){inspect $= 1$}{
Pull $l^i$ with $i$ chosen uniformly at random among the other players \;
\lIf(\tcp*[f]{lying player}){$\eta_{l^i} = 0$}{state $\gets$ ``punishing''}}
\uElse{Pull $l^j$;\
\lIf%(\tcp*[f]{player signals punishment})
{observed two collisions in a row}{state $\gets$ ``punishing''}
}}

\lIf{state $=$ ``punishing''}{$\punishsemi(\delta)$}
}
   \caption{\label{alg:algo3}\algothree}
\end{algorithm2e}

\begin{protocol}[h] 
    \DontPrintSemicolon
   \KwIn{$\pmb{\pi}, \text{blocknumber}$}
   taken\_arms $\gets \emptyset$ \;
   \For{$s = 0, \ldots, M -1$}{
   dict $\gets s + \text{blocknumber} - 1 (\text{mod } M) +1$ \tcp*{current dictator}
   $p \gets \min \{p' \in [M] \ | \ \pi^{\text{dict}}_{p'} \not\in \text{taken\_arms}\}$ \tcp*{best available choice}

   \lIf{$\pi^{\text{dict}}_{p} \neq 0$}{ $l^{\text{dict}} \gets \pi^{\text{dict}}_{p}$ and add $\pi^{\text{dict}}_{p}$ to taken\_arms}
   \lElse(\tcp*[f]{explore}){$l^{\text{dict}} \gets t+\text{dict} \ (\text{mod } K) + 1$}
   }
   comm\_arm $\gets \min [K] \setminus \text{taken\_arms}$ \;
	\Return ($l$,  comm\_arm)	   
    \caption{\label{proto:rsd}\rsd}
\end{protocol}

\begin{protocol}[h] 
    \DontPrintSemicolon
    \KwIn{$\text{blocknumber}, \text{state}, \pmb{\pi}, \text{arm\_comm}$}
        $\text{ExploitPlayers} = \{ i \in [M] \ | \ \pi^i_1 \neq 0 \}$;\quad
        $\lambda \gets \pi^{\text{blocknumber}}$ \;
    \lIf{$\lambda_1 \neq 0$}{state $\gets$ ``punishing'' \tcp*[f]{this player already sent}}
    
\lWhile{blocktime $\leq 2K$}{Pull $t+j (\text{mod } K) +1$}
	\lIf(\tcp*[f]{repeat signal}){blocktime $= 2K$}{    $\sendbit(\text{comm\_arm}, \text{ExploitPlayers},j)$ }
	\lElse{\lWhile{blocktime $\leq 4K$}{Pull $t+j (\text{mod } K) +1$}	}
		\vspace{-0.5em}
	\For{$K$ rounds}{\lIf(\tcp*[f]{signal punishment}){state $=$ ``punishing''}{Pull $j$}
	\uElse{Pull $k = t+j (\text{mod } K) + 1$ ;\quad
		\lIf{$\eta_k=1$}{state $\gets$ ``punishing'' }}}
	\vspace{0.5em}
    \For(\tcp*[f]{receive preferences}){$n=1,\ldots,MK$}
	{Pull $k = t+j (\text{mod } K) + 1$ \;
	$m \gets \left\lceil n/K \right\rceil$ \tcp*{communicating player sends her $m$-th pref.\ arm} 
	\uIf{$\eta_{k} = 1$}{
	\lIf{$\lambda_m \neq 0$}{state $\gets$ ``punishing'' \tcp*[f]{received two signals}}
	\lElse{$\lambda_m \gets k$}	
	}}
	
	\vspace{0.5em}
	\For(\tcp*[f]{repetition block}){$n=1,\ldots,M^2 K$}
	{$m \gets \left\lceil \frac{n \ (\text{mod } MK)}{K} \right\rceil$ and $l\gets \lceil \frac{n}{MK} \rceil$ \tcp*{$l$ repeats the $m$-th pref.}
	\lIf{$j=l$}{Pull $\lambda_m$ %\tcp*[f]{$j$ is repeating}
	}
	\uElse{Pull $k = t+j \ (\text{mod } K) + 1$ \;
	\lIf{$\eta_k=1$ and $\lambda_m \neq k$}{state $\gets$ ``punishing'' \tcp*[f]{info differs}}}}
	\lIf{$\card\left\lbrace \lambda_m \neq 0 \ | \ m \in [M]\right\rbrace \neq M$}{state $\gets$ ``punishing'' \tcp*[f]{did not send all}}
	\Return (state, $\lambda$)
    \caption{\label{proto:listen}\listen}
\end{protocol}

\begin{protocol}[h] 
    \DontPrintSemicolon
   \KwIn{$\text{comm\_arm}, \text{ExploitPlayers},j$}
   \lIf{$\text{ExploitPlayers} = \emptyset$}{$\widetilde{j} \gets j$}
   \lElse{$\widetilde{j} \gets \min \text{ExploitPlayers}$}
   \lFor(\tcp*[f]{send bit to exploiting players}){$K$ rounds}{Pull $t+ \widetilde{j} (\text{mod } K) +1$}
    \lFor(\tcp*[f]{send bit to exploring players}){$K$ rounds}{Pull comm\_arm}
    \caption{\label{proto:sendbit}\sendbit}
\end{protocol}

\begin{protocol}[h] 
    \DontPrintSemicolon
    \KwIn{$\pmb{\pi}, j, \text{comm\_arm}$}
    $\text{ExploitPlayers} = \{ i \in [M] \setminus \{j\} \ | \ \pi^i_1 \neq 0 \}$;\
    $\lambda \gets \pi^j$ \tcp*{$\lambda$ is signal to send}
    state $\gets$ ``exploiting'' \tcp*{state after the protocol}
    $\sendbit(\text{comm\_arm}, \text{ExploitPlayers},j)$ \tcp*{initiate communication block}
    \vspace{0.5em}
        \lFor(\tcp*[f]{wait for repetition}){$2K$ rounds}{Pull $t+j (\text{mod } K) +1$}
        \vspace{0.5em}
    \For(\tcp*[f]{receive punish signal}){$K$ rounds}{Pull $t+j (\text{mod } K) +1$; \lIf{$\eta_k=1$}{state $\gets$ ``punishing'' }}
    \vspace{0.5em}
    \lFor(\tcp*[f]{send $k$-th preferred arm}){$n=1,\ldots, MK$}{pull $\lambda_{\left\lceil \frac{n}{K} \right\rceil}$ }
	\vspace{0.5em}
	\For(\tcp*[f]{repetition block}){$n=1,\ldots,M^2 K$}
	{$m \gets \left\lceil \frac{n \ (\text{mod } MK)}{K} \right\rceil$ and $l\gets \lceil \frac{n}{MK} \rceil$ \tcp*{$l$ repeats the $m$-th pref.}
	\lIf{$j=l$}{Pull $\lambda_m$ %\tcp*[f]{$j$ is repeating}
	}
	\uElse{Pull $k = t+j \ (\text{mod } K) + 1$ \;
	\lIf{$\eta_k=1$ and $\lambda_m \neq k$}{state $\gets$ ``punishing'' \tcp*[f]{info differs}}}}
	\Return state
    \caption{\label{proto:prefsignal}\prefsignal}
\end{protocol}

\begin{protocol}[h] 
    \DontPrintSemicolon
    \KwIn{$\delta$}
    \lIf{$\text{ExploitPlayers}=[M]$}{collide with each player twice}
    \Else(\tcp*[f]{signal punishment during rounds $3K+1, \ldots, 5K$ of a block}){
        \lFor{$3K$ rounds}{Pull $t+j (\text{mod } K)+1$}
        $\sendbit(\text{comm\_arm}, \text{ExploitPlayers},j)$}
   	$\alpha \gets \left(\frac{1+\delta}{1-\delta}\right)^2 \left(1 - 1/K\right)^{M-1}$ and $\delta' = \frac{1-\alpha}{1+3\alpha}$ \;
   	Set $\hmu_k^j, S_k^j, v_k^j, n_k^j \gets 0$\;
   	\While(\tcp*[f]{estimate $\mu_k^j$}){$\exists k \in [K], \delta' \hmu_k^j < 2 s_k^j(\log(T)/n_k^j)^{1/2} + \frac{14 \log(T)}{3 (n_k^j-1)}$}{Pull $k=t+j \ (\text{mod } K) +1$ \;
   	\uIf{$\delta' \hmu_k^j < 2 s_k^j(\log(T)/n_k^j)^{1/2} + \frac{14 \log(T)}{3 (n_k^j-1)}$}{Update $\hmu_k^j \gets \frac{n_k^j}{n_k^j+1}\hmu_k^j + X_k(t)$ and $n_k^j \gets n_k^j+1$ \;
   			 Update $S_k^j \gets S_k^j + X_k^2$ and $s_k^j \gets \sqrt{\frac{S_k^j - (\hmu_k^j)^2}{n_k^j-1}}$   			 }
   	
   	}
  $p_k \gets \bigg( 1 - \Big(\alpha\frac{\sum_{l=1}^M \hmu^j_{(l)}(t)}{M \hmu^j_k(t)}\Big)^{\frac{1}{M-1}}\bigg)_+$;\quad
   	$\widetilde{p}_k \gets p_k/\sum_{l=1}^K p_l$ \tcp*{renormalize}
   	\lWhile(\tcp*[f]{punish}){$t \leq T$}{Pull $k$ with probability $p_k$}
    \caption{\label{proto:punishsemi}\punishsemi}
\end{protocol}

\paragraph{Initialization phase.} \algothree starts with the exact same initialization as \algotwo[,]which is given by Protocol~\ref{proto:init}, to estimate $M$ and attribute ranks among the players. Afterwards, they start the exploration.

\medskip

In the remaining of the algorithm, as already explained in Section~\ref{sec:algo3}, the time is divided into superblocks, which are divided into $M$ blocks of length $5K+MK+M^2K$. During the $j$-th block of a superblock, the dictators ordering for RSD is $(j, \ldots, M, 1, \ldots, j-1)$. Moreover, only the $j$-th player can send messages during this block if she is still exploring.

\paragraph{Exploration.} The exploiting players sequentially pull all the arms in $[K]$ to avoid collisions with any other exploring player. Yet, they still collide with exploiting players.

\algothree is designed so that all players know at each round the $M$ preferred arms of any exploiting players and their order. The players thus know which arms are occupied by the exploiting players during a block $j$. The communication arm is thus a common arm unoccupied by any exploiting player. When an exploring player encounters a collision on this arm at the beginning of the block, this means that another player signaled the start of a communication block. In that case, the exploring player starts \listen[,]described by Protocol~\ref{proto:listen}, to receive the messages of the communicating player.

On the other hand, when an exploring player $j$ knows her $M$ preferred arms and their order, she waits for the next block $j$ to initiate communication. She then proceeds to \prefsignal[,]given by Protocol~\ref{proto:prefsignal}.

\paragraph{Communication block.}

In a communication block, the communicating player first collides with each exploiting and exploring player to signal them the start of a communication block as described by Protocol~\ref{proto:sendbit}. These collisions need to be done in a particular way given by \sendbit so that all players correctly detect the start of a communication block. These players then repeat this signal to ensure that every player is listening.

The communicating player then sends to all players her $M$ preferred arms in order of preferences.
Afterwards, each player repeats this list to ensure that no malicious player interfered during communication. As soon as some malicious behavior is observed, the start of \punishsemi[,]given by Protocol~\ref{proto:punishsemi}, is signaled to all players.

\paragraph{Exploitation.}
An exploiting player starts any block $j$ by computing the attribution of the RSD algorithm between the exploiting players given their known preferences and the dictatorship ordering $(j, \ldots, j-1)$. She then pulls her attributed arm for the whole block, unless she receives a signal.

A signal is received when she collides with an exploring player, while unintended\footnote{She normally collides with exploring players. Yet as she knows the set of exploring players, she exactly knows when this happens.}. If it is at the beginning of a block, it means that a communication block starts. Otherwise, she just enters the punishment protocol. Note that the punishment protocol starts by signaling the start of \punishsemi to ensure that every cooperative player starts punishing.

\medskip

Another security is required to ensure that the selfish player truthfully reports her preferences. She could otherwise report fake preferences to decrease another player's utility while her best arm remains uncontested and thus available. To avoid this, \algothree uses \textit{random inspections} when all players are exploiting. With probability $\sqrt{\log(T)}/T$ at each round, any player checks that some other player is indeed exploiting the arm she is attributed by the RSD algorithm. 
If it is not the case, the inspecting player signals the start of \punishsemi to everyone by colliding twice with everybody, since a single collision could be a random inspection.
Because of this, the selfish player can not pull another arm than the attributed one too often without starting a punishment scheme. Thus, if she did not report her preferences truthfully, this also has a cost for her.
\newpage
\subsection{Regret analysis} \label{app:rsdregretproof}

This section aims at proving the first point of Theorem~\ref{thm:rsd}. 
\algothree uses the exact same initialization phase as \algotwo[,]and its guarantees are thus given by Lemma~\ref{lemma:initfull}. Here again, the regret is decomposed into three parts:
\begin{equation}
\label{eq:regdec3}
R_T^{\text{RSD}} = R^{\text{init}} + R^{\text{comm}} + R^{\text{explo}},
\end{equation}
\begin{equation*}
\text{where } \left\{ \begin{split} \begin{aligned} & R^{\text{init}} = T_{\text{init}} {\mathlarger \mE_{\sigma \sim \cU\left(\mathfrak{S}_M\right)} } \bigg[ \sum_{k = 1}^{M} \mu_{\pi_\sigma(k)}^{\sigma(k)}\bigg] - \mathbb{E}_\mu \Big[{\mathlarger\sum_{t=1}^{T_{\text{init}}}} {\mathlarger \sum_{j = 1}^M} r^j(t) \Big]  %\hspace{3.2cm} 
\text{ with } T_{\text{init}} = (12eK^2 + K) \log(T), \\
& R^{\text{comm}} = \card \text{Comm}{\mathlarger \mE_{\sigma \sim \cU\left(\mathfrak{S}_M\right)} } \bigg[ \sum_{k = 1}^{M} \mu_{\pi_\sigma(k)}^{\sigma(k)}\bigg]- \mathbb{E}_\mu \Big[{\mathlarger\sum_{t \in \text{Comm}}}{\mathlarger \sum_{j=1}^M}  r^j(t)) \Big], \\
& R^{\text{explo}} = \card \text{Explo} {\mathlarger \mE_{\sigma \sim \cU\left(\mathfrak{S}_M\right)} } \bigg[ \sum_{k = 1}^{M} \mu_{\pi_\sigma(k)}^{\sigma(k)}\bigg] - \mathbb{E}_\mu \Big[{\mathlarger\sum_{t \in \text{Explo}}}{\mathlarger \sum_{j=1}^M} r^j(t)) \Big] \end{aligned} \end{split} \right.
\end{equation*}
with Comm defined as all the rounds of a block where at least a cooperative player uses \listen protocol and $\text{Explo} = \{T_{\text{init}} +1, \ldots, T \} \setminus \text{Comm}$.
In case of a successful initialization, a single player can only initiate a communication block once without starting a punishment protocol. Thus, as long as no punishment protocol is started:
$
\card\text{Comm} \leq M(5K + M K + M^2 K) = \cO(M^3K).
$

Denote by $\Delta^j = \min_{k \in [M]} \mu_{(k)}^j - \mu_{(k+1)}^j$ the level of precision required for player $j$ to know her $M$ preferred arms and their order. Proposition~\ref{prop:stopexplor2} gives the exploration time required for any player $j$:

\begin{prop}\label{prop:stopexplor2}
With probability $1-\cO\left( \frac{K}{T} \right)$ and as long as no punishment protocol is started, the player $j$ starts exploiting after at most $\cO\left( \frac{K \log(T)}{(\Delta^j)^2} + M^3K \right)$ exploration pulls.
\end{prop}
\begin{proof}
In the following, the initialization is assumed to be successful, which happens with probability $1-\cO\left( \frac{M}{T} \right)$.
Moreover, Hoeffding inequality yields: $$\mP\left[ \forall t \leq T,  \left| \hmu_k^j(t) - \mu_k^j(t) \right| \geq \sqrt{\frac{2 \log(T)}{T_k^j(t)}} \right] \leq \frac{2}{T}$$ where $T_k^j(t)$ is the number of exploratory pulls on arm $k$ by player $j$. With probability~${1-\cO\left( \frac{K}{T} \right)}$, player $j$ then correctly estimates all arms at each round. The remaining of the proof is conditioned on this event.

During the exploration, player $j$ sequentially pulls the arms in $[K]$. Denote by $n$ the smallest integer such that $\sqrt{\frac{2 \log(T)}{n}} \leq 4 \Delta^j$. 
It directly comes that $n = \cO\left( \frac{\log(T)}{(\Delta^j)^2} \right)$.
Under the considered events, player $j$ then has determined her $M$ preferred arms and their order after $K n$ exploratory pulls. Moreover, she needs at most $M$ blocks before being able to initiate her communication block and starts exploiting. Thus, she needs at most $\cO\left( \frac{K \log(T)}{(\Delta^j)^2} + M^3K \right)$ exploratory pulls, leading to Proposition~\ref{prop:stopexplor2}.
\end{proof}

\begin{proof}[Proof of the first point of Theorem~\ref{thm:rsd}.]
Assume all players play \algothree[.]Simply by bounding the size of the initialization and the communication phases, it comes:
\begin{equation*}
R^{\text{init}} + R^{\text{comm}} \leq \cO\left( MK^2 \log(T)\right).
\end{equation*}
Proposition~\ref{prop:stopexplor2} yields that with probability $1-\cO\left( \frac{KM}{T} \right)$, all players start exploitation after at most $\cO\left(\frac{K \log(T)}{\Delta^2} \right)$ exploratory pulls. 

\medskip

For $p=\sqrt{\log(T)}/T$, with probability $\cO(p^2M)$ at any round $t$, a player is inspecting another player who is also inspecting or a player receives two consecutive inspections. These are the only ways to start punishing when all players are cooperative.
As a consequence, when all players follow \algothree[,]they initiate the punishment protocol with probability $\cO\left(p^2MT\right)$. Finally, the total regret due to this event grows as~$\cO\left(M^2 \log(T)\right)$.

\medskip

If the punishment protocol is not initiated, players cycle through the RSD matchings of $\sigma \circ \sigma^{-1}_0, \ldots, \sigma \circ \sigma^{-M}_0$ where $\sigma_0$ is the classical $M$-cycle and $\sigma$ is the players permutation returned by the initialization. 
Define $U(\sigma)=\sum_{k = 1}^{M}\mu_{\pi_\sigma(k)}^{\sigma(k)},$ where $\pi_\sigma(k)$ is the arm attributed to the $k$-th dictator, $\sigma(k)$, as defined in Section~\ref{sec:randomass}. $U(\sigma)$ is the social welfare of RSD algorithm when the dictatorships order is given by the permutation $\sigma$. As players all follow \algothree here, $\sigma$ is chosen uniformly at random in $\mathfrak{S}_M$ and any $\sigma \circ \sigma^{-k}_0$ as well. Then $$\mE_{\sigma \sim \cU\left(\mathfrak{S}_M\right)}\left[ \frac{1}{M} \sum_{k=1}^M U(\sigma \circ \sigma^{-M}_0) \right] = \mE_{\sigma \sim \cU\left(\mathfrak{S}_M\right)}\left[U(\sigma) \right].$$

This means that in expectation, the utility given by the exploitation phase is the same as the utility of the RSD algorithm when choosing a permutation uniformly at random. Considering the low probability event of a punishment protocol, an unsuccesful initialization or a bad estimation of an arm finally yields:
\begin{equation*}
R^{\text{explo}} \leq \cO\left(\frac{MK \log(T)}{\Delta^2}  \right)\ .
\end{equation*}
Equation~\eqref{eq:regdec3} concludes the proof.
\end{proof}

\subsection{Selfish-robustness of \algothree} \label{app:rsdgreedyproof}

In this section, we prove the two last points of Theorem~\ref{thm:rsd}. Three auxiliary Lemmas are first needed. They are proved in Appendix~\ref{app:auxlemmasrsd}.

\begin{enumerate}
\item Lemma~\ref{lemma:meancomparison} compares the utility received by player $j$ from the RSD algorithm with the utility given by sequentially pulling her $M$ best arms in the $\delta$-heterogeneous setting. 
\item Lemma~\ref{lemma:punishhetero} gives an equivalent version of Lemma~\ref{lemma:punishhomo}, but for the $\delta$-heterogeneous setting.
\item Lemma~\ref{lemma:rsdmatching} states that the expected utility of the assignment of any player during the exploitation phase does not depend on the strategy of the selfish player. The intuition behind this result is already given in Section~\ref{sec:algo3}.

In the case of several selfish players, they could actually fix the joint distribution of~$(\sigma^{-1}(j), \sigma^{-1}(j'))$. A simple rotation with a $M$-cycle is then not enough to recover a uniform distribution over $\mathfrak{S}_M$ in average. A more complex rotation is then required and the dependence in $M$ would blow up with the number of selfish players.
\end{enumerate}
\begin{lemm} \label{lemma:meancomparison}
In the $\delta$-heterogeneous case for any player $j$ and permutation $\sigma$:
\begin{small}
\begin{equation*}
\frac{1}{M}\sum_{k=1}^M  \mu_{(k)}^j \leq \widetilde{U}_j(\sigma) \leq \frac{(1+\delta)^2}{(1-\delta)^2 M}\sum_{k=1}^M  \mu_{(k)}^j,
\end{equation*}
\end{small}
where $\widetilde{U}_j(\sigma) \coloneqq\frac{1}{M} \sum_{k=1}^M \mu_{\pi_{\sigma \circ \sigma^{-k}_0}\left(\sigma_0^k \circ \sigma^{-1}(j)\right)}^j$.
\end{lemm}

Following the notation of Section~\ref{sec:randomass}, $\pi_\sigma(\sigma^{-1}(j))$ is the arm attributed to player $j$ by RSD when the dictatorship order is given by $\sigma$. $\widetilde{U}_j(\sigma)$ is then the average utility of the exploitation when $\sigma$ is the permutation given by the initialization.

\begin{lemm} \label{lemma:punishhetero}
Recall that $\gamma=(1-1/K)^{M-1}$. In the $\delta$-heterogeneous setting with $\delta < \frac{1-\sqrt{\gamma}}{1+\sqrt{\gamma}}$, if the punish protocol is started at time $T_{\text{punish}}$ by $M-1$ players, then for the remaining player $j$, independently of her sampling strategy:
\begin{small}
\begin{equation*}
\mE[\rew_T^j | \text{punishment}] \leq \mE[\rew_{T_{\text{punish}} + t_p}^j] + \widetilde{\alpha} \frac{T - T_{\text{punish}} - t_p}{M} \sum_{k=1}^M \mu^j_{(k)},
\end{equation*}
\end{small}
with $t_p = \cO\left(\frac{K\log(T)}{(1-\delta)(1-\widetilde{\alpha})^2 \mu_{(K)}} \right)$ and $\widetilde{\alpha}=\frac{1+\left(\frac{1+\delta}{1-\delta}\right)^2\gamma}{2}$.
\end{lemm}

\begin{lemm} \label{lemma:rsdmatching}
The initialization phase is successful when all players end with different ranks in $[M]$. For any player $j$, independently of the behavior of the selfish player:
\begin{small}
\begin{equation*}
\mE_{\sigma \sim \text{successful initialization}}\left[ \widetilde{U}_j(\sigma) \right] = \mE_{\sigma \sim \cU\left(\mathfrak{S}_M\right)}\left[\mu_{\pi_\sigma(\sigma^{-1}(j))}^j \right].
\end{equation*}
\end{small}
where $\widetilde{U}_j(\sigma)$ is defined as in Lemma~\ref{lemma:meancomparison} above.
\end{lemm}

\begin{proof}[Proof of the second point of Theorem~\ref{thm:rsd} (Nash equilibrium).]
First fix $T_{\text{punish}}$ the beginning of the punishment protocol. Note $s$ the profile where all players follow \algothree and $s'$ the individual strategy of the selfish player $j$.

As in the homogeneous case, the player earns at most $T_{\text{init}} + \card\text{Comm}$ during both initialization and communication. She can indeed choose her rank at the end of the initialization, but this has no impact on the remaining of the algorithm (except for a $M^3K$ term due to the length of the last uncompleted superblock), thanks to Lemma~\ref{lemma:rsdmatching}.

With probability $1-\cO\left(\frac{KM + M\log(T)}{T}\right)$, the initialization is successful, the arms are correctly estimated and no punishment protocol is due to unfortunate inspections (as already explained in Section~\ref{app:rsdregretproof}). The following is conditioned on this event.

\medskip

Proposition~\ref{prop:stopexplor2} holds independently of the strategy of the selfish player. Moreover, the exploiting players run the RSD algorithm only between the exploiters. This means that when all cooperative players are exploiting, if the selfish player did not signal her preferences, she would always be the last dictator in the RSD algorithm. Because of this, it is in her interest to report as soon as possible her preferences.

Moreover, reporting truthfully is a dominant strategy for the RSD algorithm, meaning that when all players are exploiting, the expected utility received by the selfish player is at most the utility she would get by reporting truthfully. As a consequence, the selfish player can improve her expected reward by at most the length of a superblock during the exploitation phase. Wrapping up all of this and defining $t_0$ the time at which all other players start exploiting:
\begin{small}
\begin{equation*}
\mE\left[\rew_{T_\text{punish}+t_p}^j(s', s_{-j}) \right] \leq t_0 + (T_{\text{punish}}+t_p-t_0) \mE_{\sigma \sim \cU(\mathfrak{S}_M)}\left[ \mu^j_{\pi_\sigma(\sigma^{-1}(j))} \right] + \cO(M^3K).
\end{equation*}
\end{small}
with $t_0 = \cO \left(\frac{K \log(T)}{\Delta^2} + K^2 \log(T)\right)$. %The last term corresponds to the length of a superblock, since the rank of a player actually has an influence on the reward due to the last superblock, which is not completed.
Lemma~\ref{lemma:punishhetero} then yields for $\widetilde{\alpha}=\frac{1+\left(\frac{1+\delta}{1-\delta}\right)^2\alpha}{2}$:
\begin{small}
\begin{equation*}
\mE\left[\rew_{T}^j(s',s_{-j}) \right] \leq t_0 + (T_{\text{punish}}+t_p-t_0) \mE_{\sigma \sim \cU(\mathfrak{S}_M)}\left[ \mu^j_{\pi_\sigma(\sigma^{-1}(j))} \right] + \widetilde{\alpha}\frac{T- T_{\text{punish}} - t_p}{M} \sum_{k=1}^M \mu_{(k)}^j + \cO(M^3K).
\end{equation*}
\end{small}

Thanks to Lemma~\ref{lemma:meancomparison}, $\mE_{\sigma \sim \cU(\mathfrak{S}_M)}\left[ \mu^j_{\pi_\sigma(\sigma^{-1}(j))} \right] \geq \frac{\sum_{k=1}^M \mu_{(k)}^j}{M}$.
We assume $\delta < \frac{1-(1-1/K)^{\frac{M-1}{2}}}{1+(1-1/K)^{\frac{M-1}{2}}}$ here, so that $\widetilde{\alpha} < 1$. Because of this, the right term is maximized when $T_{\text{punish}}$ is maximized, \ie equal to $T$. Then:

\begin{equation*}
\mE\left[\rew_{T}^j(s',s_{-j}) \right] \leq T \mE_{\sigma \sim \cU(\mathfrak{S}_M)}\left[ \mu^j_{\pi_\sigma(\sigma^{-1}(j))} \right] + t_0 + t_p + \cO(M^3K).
\end{equation*}

Using the first point of Theorem~\ref{thm:rsd} to compare $T \mE_{\sigma \sim \cU(\mathfrak{S}_M)}\left[ \mu^j_{\pi_\sigma(\sigma^{-1}(j))} \right]$ with $\rew_T^j(s)$ and adding the low probability event then yields the first point of Theorem~\ref{thm:rsd}.
\end{proof}

\begin{proof}[Proof of the second point of Theorem~\ref{thm:rsd} (stability).]
% a voir, mais a tous les coups, la possibilité de signaler autre chose que ses vraies préférences a du etre oubliée ici
For $p_0=\cO\left(\frac{KM + M\log(T)}{T}\right)$, with probability at least $1-p_0$, the initialization is successful, the cooperative players start exploiting with correct estimated preferences after a time at most $t_0 = \cO\left( K^2 \log(T) + \frac{K \log(T)}{\Delta^2}\right)$ and no punishment protocol is started due to unfortunate inspections. 
% the following will be conditioned on that event
Define $\varepsilon' = t_0 + Tp_0 + 7M^3K$.
Assume that the player $j$ is playing a deviation strategy $s'$ such that for some $i$ and $l>0$:
\begin{small} 
\begin{equation*}
\mE\left[ \rew_T^i(s', s_{-j})\right] \leq \mE\left[ \rew_T^i(s)\right] - l - \varepsilon'
\end{equation*}\end{small}
First, let us fix $\sigma$ the permutation returned by the initialization, $T_{\text{punish}}$ the time at which the punishment protocol starts and divide $l = l_{\text{before punishment}} + l_{\text{after punishment}}$ in two terms: the regret incurred before the punishment protocol and the regret after. Let us now compare $s'$ with $s^*$, the optimal strategy for player $j$. Let $\varepsilon$ take account of the low probability event of a bad initialization/exploration, the last superblock that remains uncompleted, the time before all cooperative players start the exploitation and the event that a punishment accidentally starts. Thus the only way for player $i$ to suffer some additional regret before punishment is to lose it during a completed superblock of the exploitation.
Three cases are possible:
\begin{enumerate}[wide, labelwidth=!, labelindent=0pt]
\item The selfish player truthfully reports her preferences. The average utility of player $i$ during the exploitation is then $\widetilde{U}_i(\sigma)$ as defined in Lemma~\ref{lemma:rsdmatching}.
The only way to incur some additional loss to player $i$ before the punishment is then to collide with her, in which case her loss is at most $(1+\delta)\mu_{(1)}$ while the selfish player's loss is at least $(1-\delta) \mu_{(M)}$.

\medskip

After $T_{\text{punish}}$, Lemma~\ref{lemma:punishhetero} yields that the selfish player suffers a loss at least $(1-\widetilde{\alpha})\frac{T-T_{\text{punish}}-t_p}{M}\sum_{k=1}^M \mu_{(k)}^j$, while any cooperative player $i$ suffers a loss at most $(T-T_\text{punish}) \widetilde{U}_i(\sigma)$. Thanks to Lemma~\ref{lemma:meancomparison} and the $\delta$-heterogeneity assumption, this term is smaller than
$\frac{T-T_{punish}}{M}\left(\frac{1+\delta}{1-\delta} \right)^3\sum_{k=1}^M \mu_{(k)}^j$.

Then, the selfish player after $T_{\text{punish}}$ suffers a loss at least $\frac{(1-\widetilde{\alpha})(1-\delta)^3}{(1+\delta)^3} l_{\text{after punish}} - t_p$.

\medskip

In the first case, we thus have for $\beta = \min(\frac{(1-\widetilde{\alpha})(1-\delta)^3}{(1+\delta)^3}, \frac{(1-\delta)\mu_{(M)}}{(1+\delta)\mu_{(1)}})$: 
\begin{equation*}
\mE[\rew_T^j(s',s_{-j}) | \sigma] \leq \mE[\rew_T^j(s^*, s_{-j}) | \sigma] - \beta l + t_p.
\end{equation*}

\item The selfish player never reports her preferences. In this case, it is obvious that the utility returned by the assignments to any other player is better than if the selfish player reports truthfully. Then the only way to incur some additional loss to player $i$ before punishment is to collide with her, still leading to a ratio of loss at most $\frac{\mu_{(M)}^j}{\mu_{(1)}^i}$.

\medskip

From there, it can be concluded as in the first case that for $\beta = \min(\frac{(1-\widetilde{\alpha})(1-\delta)^3}{(1+\delta)^3}, \frac{(1-\delta)\mu_{(M)}}{(1+\delta)\mu_{(1)}})$:

\begin{equation*}
\mE[\rew_T^j(s',s_{-j}) | \sigma] \leq \mE[\rew_T^j(s^*, s_{-j}) | \sigma] - \beta l + t_p.
\end{equation*}

\item The selfish player reported fake preferences. If these fake preferences never change the issue of the \rsd protocol, this does not change from the first case. Otherwise, for any block where the final assignment is changed, the selfish player does not receive the arm she would get if she reported truthfully.
Denote by $n$ the number of such blocks, by $N_{\text{lie}}$ the number of times player $j$ did not pull the arm attributed by \rsd during such a block before $T_{\text{punish}}$ and by $l_{b}$ the loss incurred to player $i$ on the other blocks. 

As for the previous cases, the loss incurred by the selfish player during the blocks where the assignment of \rsd is unchanged is at least $\frac{(1-\delta)\mu_{(M)}}{(1+\delta)\mu_{(1)}} l_{b}$.

\medskip

Each time the selfish player pulls the attributed arm by \rsd in a block where the assignment is changed, she suffers a loss at least $\Delta$. %We can assume w.l.o.g. that the number of exploiting timesteps in these blocks before $T_{\text{punish}}$ is $\frac{k}{M}\left(T_{\text{punish}}-t_0\right)$. 
The total loss for the selfish player is then (w.r.t. the optimal strategy $s^*$) at least:
\begin{small}
\begin{equation*}
(1-\widetilde{\alpha}) \frac{T-T_{\text{punish}}-t_p}{M} \sum_{k=1}^M \mu_{(k)}^j +  \left(\frac{n}{M}\left(T_{\text{punish}}-t_0\right) - N_{\text{lie}} \right) \Delta +\frac{(1-\delta)\mu_{(M)}}{(1+\delta)\mu_{(1)}} l_{b}.
\end{equation*}
\end{small}

On the other hand, the loss for a cooperative player is at most:
\begin{small}
\begin{equation*}
\frac{T-T_{\text{punish}}}{M}\left(\frac{1+\delta}{1-\delta}\right)^3\sum_{k=1}^M \mu_{(k)}^j + \frac{n}{M}(T_{\text{punish}}-t_0)(1+\delta)\mu_{(1)} + l_b.
\end{equation*}
\end{small}

Moreover, each time the selfish player does not pull the attributed arm by \rsd[,]she has a probability $\widetilde{p}=1-(1-\frac{p}{M-1})^{M-1} \geq \frac{p}{2}$ for $p=\frac{\sqrt{\log(T)}}{T}$, to receive a random inspection and thus to trigger the punishment protocol. Because of this, $N_{\text{lie}}$ follows a geometric distribution of parameter $\widetilde{p}$ and $\mE[N_{\text{lie}}] \leq \frac{2}{p}$. 

When taking the expectations over $T_{\text{punish}}$ and $N_{\text{lie}}$, but still fixing $\sigma$ and $n$, we get:
\begin{align*}
l_{\text{selfish}} \geq (1-\widetilde{\alpha}) \frac{T-\mE[T_{\text{punish}}]-t_p}{M} \sum_{k=1}^M \mu_{(k)}^j +  \left(\frac{n}{M}\left(\mE[T_{\text{punish}}]-t_0\right) - 2/p \right) \Delta +\frac{(1-\delta)\mu_{(M)}}{(1+\delta)\mu_{(1)}} l_{b}, 
\\
l \leq \frac{T-\mE[T_{\text{punish}}]}{M}\left(\frac{1+\delta}{1-\delta}\right)^3\sum_{k=1}^M \mu_{(k)}^j + \frac{n}{M}(\mE[T_{\text{punish}}]-t_0)(1+\delta)\mu_{(1)} + l_b.
\end{align*}

First assume that $\frac{n}{M}(\mE[T_{\text{punish}}] -t_0) \geq \frac{4}{p}$. In that case, we get:

\begin{align*}
l_{\text{selfish}} \geq (1-\widetilde{\alpha}) \frac{T-\mE[T_{\text{punish}}]-t_p}{M} \sum_{k=1}^M \mu_{(k)}^j +  \frac{n}{2M}(\mE[T_{\text{punish}}]-t_0) \Delta +\frac{(1-\delta)\mu_{(M)}}{(1+\delta)\mu_{(1)}} l_{b},\\
l \leq \frac{T-\mE[T_{\text{punish}}]}{M}\left(\frac{1+\delta}{1-\delta}\right)^3\sum_{k=1}^M \mu_{(k)}^j + \frac{n}{M}(\mE[T_{\text{punish}}]-t_0)(1+\delta)\mu_{(1)} + l_b.
\end{align*}

In the other case, we have by noting that $(1+\delta)\mu_{(1)} \leq \frac{1+\delta}{1-\delta}\sum_{k=1}^M \mu_{(k)}^j$:
\begin{align*}
l_{\text{selfish}} \geq (1-\widetilde{\alpha}) T\left( 1-\frac{4M}{\sqrt{\log(T)}} -t_p \right)\frac{1}{M} \sum_{k=1}^M \mu_{(k)}^j + \frac{(1-\delta)\mu_{(M)}}{(1+\delta)\mu_{(1)}} l_{b},\\
l \leq T\left(1+\frac{4M}{\sqrt{\log(T)}}\right)\frac{1}{M}\left(\frac{1+\delta}{1-\delta}\right)^3\sum_{k=1}^M \mu_{(k)}^j + l_b.
\end{align*}

In any of these two cases, for $\widetilde{\beta} = \min\left((1-\widetilde{\alpha})\left( \frac{1+\delta}{1-\delta}\right)^3 \frac{\sqrt{\log(T)}-4M}{\sqrt{\log(T)}+4M}; \frac{\Delta}{(1+\delta)\mu_{(1)}};  \frac{(1-\delta)\mu_{(M)}}{(1+\delta)\mu_{(1)}}\right)$:
\begin{align*}
l_{\text{selfish}}\geq \widetilde{\beta} l - t_p 
\end{align*}

%Meanwhile, the cooperative player lost at most $\min\left(\frac{1}{M}\left(\frac{1+\delta}{1-\delta} \right)^3\sum_{k=1}^M \mu_{(k)}^j, (1+\delta)\mu_{(1)} \right)T$. So that for $\widetilde{\beta} = \min(\frac{(1-\widetilde{\alpha})(1-\delta)^3}{(1+\delta)^3}, \frac{\Delta}{M\mu_{(1)} (1+\delta)})(1-1/\sqrt{\log(T)})$:
%
%\begin{equation*}
%\mE[\rew_T^j(s_{s', j}) | \sigma] \leq \mE[\rew_T^j(s_{s^*, j}) | \sigma] - \widetilde{\beta} l + t_p + t_0.
%\end{equation*}

\end{enumerate}
\bigskip

Let us now gather all the cases. When taking the previous results in expectation over $\sigma$, this yields for the previous definition of $\widetilde{\beta}$:
\begin{equation*}
\mE[\rew_T^i(s', s_{-j})] \leq \mE[\rew_T^i(s)] - l - \varepsilon' \implies \mE[\rew_T^j(s', s_{-j})] \leq \mE[\rew_T^j(s^*, s_{-j})] - \widetilde{\beta} l + t_p + t_0.
\end{equation*}

Moreover, thanks to the second part of Theorem~\ref{thm:rsd}, $\mE[\rew_T^j(s^*, s_{-j})] \leq \mE[\rew_T^j(s)] + \varepsilon$, with $\varepsilon = \cO\left(\frac{K \log(T)}{\Delta^2} + K^2\log(T)+ \frac{K\log(T)}{(1-\delta)r^2 \mu_{(K)}}\right)$. Then by defining $l_1 = l + \varepsilon'$, $\varepsilon_1 = \varepsilon + t_p + t_0 + \widetilde{\beta} \varepsilon' = \cO(\varepsilon)$, we get:
\begin{equation*}
\mE[\rew_T^i(s', s_{-j})] \leq \mE[\rew_T^i(s)] - l_1 \implies \mE[\rew_T^j(s', s_{-j})] \leq \mE[\rew_T^j(s)] - \widetilde{\beta} l_1 + \varepsilon_1.
\end{equation*}

\end{proof}
\subsubsection{Auxiliary lemmas} \label{app:auxlemmasrsd}

\begin{proof}[Proof of Lemma~\ref{lemma:meancomparison}.]
Assume that player $j$ is the $k$-th dictator for an RSD assignment. Since only $k-1$ arms are reserved before she chooses, she earns at least $\mu_{(k)}^j$ after this assignment. This yields the first inequality:
\begin{equation*}
\widetilde{U}_j(\sigma) \geq \frac{\sum_{k=1}^M \mu_{(k)}^j}{M}
\end{equation*}

\medskip

Still assuming that player $j$ is the $k$-th dictator, let us prove that she earns at most $\left(\frac{1+\delta}{1-\delta}\right)^2 \mu_{(k)}^j$. Assume w.l.o.g. that she ends up with the arm $l$ such that $\mu_l^j > \mu_{(k)}^j$. This means that a dictator $j'$ before her preferred an arm $i$ to the arm $l$ with $\mu_l^j > \mu_{(k)}^j \geq \mu_i^j$.

Since $j'$ preferred $i$ to $l$, $\mu_i^{j'} \geq \mu_l^{j'}$. Using the $\delta$-heterogeneity assumption, it comes:
$$\mu_l^j  \leq \frac{1+\delta}{1-\delta} \mu_l^{j'} 
 \leq \frac{1+\delta}{1-\delta} \mu_i^{j'} 
 \leq \left(\frac{1+\delta}{1-\delta}\right)^2 \mu_i^{j} 
 \leq \left(\frac{1+\delta}{1-\delta}\right)^2 \mu_{(k)}^{j}
$$
Thus, player $j$ earns at most $ \left(\frac{1+\delta}{1-\delta}\right)^2 \mu_{(k)}^{j}$ after this assignment, which yields the second inequality of Lemma~\ref{lemma:meancomparison}.

\end{proof}

\begin{proof}[Proof of Lemma~\ref{lemma:punishhetero}.]
The punishment protocol starts for all cooperative players at $T_{\text{punish}}$. Define $\alpha' = \left(\frac{1+\delta}{1-\delta}\right)^2\gamma$ and $\delta' = \frac{1-\alpha'}{1+3\alpha'}$. The condition $r>0$ is equivalent to $\delta' > 0$.

As in the homogeneous case, each player then estimates each arm such that after $t_p = \cO\left( \frac{K \log(T)}{(1-\delta) \cdot (\delta')^2 \mu_{(K)}} \right)$\footnote{The $\delta$-heterogeneous assumption is here used to say that $\frac{1}{\mu_{(K)}^j} \leq \frac{1}{(1-\delta)\mu_{(K)}}$.} rounds, $(1-\delta')\hmu_k^j \leq \mu_k^j \leq (1+\delta) \hmu_k^j$ with probability $1-\cO\left(KM/T \right)$, thanks to Lemma~\ref{lemma:multiplicativeestim}. This implies that for any cooperative player $j'$:
\begin{align*}
1-p_k^{j'} & \leq \left( \gamma \frac{(1+\delta')\sum_{m=1}^M \mu_{(m)}^{j'} }{(1-\delta')M\mu_k^{j'} } \right)^{\frac{1}{M-1}} \\
& \leq \left( \gamma \frac{1+\delta'}{1-\delta'} \left(\frac{1+\delta}{1-\delta}\right)^2 \frac{\sum_{m=1}^M \mu_{(m)}^{j} }{M\mu_k^{j} } \right)^{\frac{1}{M-1}}
\end{align*}
The last inequality is due to the fact that in the $\delta$-heterogeneous setting, $\frac{\mu_k^j}{\mu_k^{j'}} \in [\left(\frac{1-\delta}{1+\delta}\right)^2, \left(\frac{1+\delta}{1-\delta}\right)^2]$. Thus, the expected reward that gets the selfish player $j$ by pulling $k$ after the time $T_{\text{punish}} + t_p$ is smaller than $\gamma \frac{1+\delta'}{1-\delta'} \left(\frac{1+\delta}{1-\delta}\right)^2 \frac{\sum_{m=1}^M \mu_{(m)}^{j} }{M }$.

Note that $\gamma \frac{1+\delta'}{1-\delta'} \left(\frac{1+\delta}{1-\delta}\right)^2 = \widetilde{\alpha}$. Considering the low probability event of bad estimations of the arms adds a constant term that can be counted in $t_p$, leading to Lemma~\ref{lemma:punishhetero}.
\end{proof}

\begin{proof}[Proof of Lemma~\ref{lemma:rsdmatching}.]
Consider the selfish player $j$ and denote $\sigma$ the permutation given by the initialization. The rank of player $j'$ is then $\sigma^{-1}(j')$. All other players $j$ pull uniformly at random until having an attributed rank. Moreover, player $j$ does not know the players with which she collides. This implies that she can not correlate her rank with the rank of a specific player, \ie $\mP_{\sigma}\left[ \sigma(k')=j' | \sigma(k)=j \right]$ does not depend on $j'$ as long as $j' \neq j$.

\medskip

This directly implies that the distribution of $\sigma | \sigma(k)=j$ is uniform over $\mathfrak{S}_{M}^{j \to k}$. Thus, the distribution of $ \sigma \circ \sigma_0^{-l} | \sigma(k)=j$ is uniform over $\mathfrak{S}_{M}^{j \to k+l \ (\text{mod } M)}$ and finally for any $j' \in [M]$:
\begin{align*}
\mE_{\sigma \sim \text{successful initialization}}\left[ \frac{1}{M} \sum_{l=1}^M \mu_{\pi_{\sigma \circ \sigma_0^{-l}}\left(\sigma_0^l \circ \sigma^{-1}(j)\right)}^j \ \bigg| \ \sigma(k)=j \right] & = \frac{1}{M} \sum_{l=1}^M \mE_{\sigma \sim \cU\left(\mathfrak{S}_{M}^{j \to l}\right)}\left[\mu_{\pi_\sigma(\sigma^{-1}(j'))}^{j'} \right], \\
& = \frac{1}{M} \sum_{l=1}^M  \frac{1}{(M-1)!}\sum_{\sigma \in \mathfrak{S}_{M}^{j \to l}}  \mu_{\pi_\sigma(\sigma^{-1}(j'))}^{j'}, \\
& = \frac{1}{M!}\sum_{\sigma \in \mathfrak{S}_{M}}  \mu_{\pi_\sigma(\sigma^{-1}(j'))}^{j'}.
\end{align*}
Taking the expectation of the left term then yields Lemma~\ref{lemma:rsdmatching}.
\end{proof}
\end{document}